%% file: example_paper.tex
\documentclass{article}

\usepackage{microtype}
\usepackage{graphicx}
\usepackage{subcaption}
\usepackage{booktabs} % for professional tables
\usepackage{enumitem} % for customizing lists

\usepackage{hyperref}

\usepackage[preprint]{icml2026}

\usepackage{amsmath, bm}
\usepackage{amssymb}
\usepackage{mathtools}
\usepackage{amsthm}

\usepackage{multirow}

\usepackage[capitalize,noabbrev]{cleveref}

\theoremstyle{plain}
\newtheorem{theorem}{Theorem}[section]
\newtheorem{proposition}[theorem]{Proposition}

\theoremstyle{definition}

\theoremstyle{remark}

\newcommand{\eg}{\emph{e. g.}}

\usepackage[textsize=tiny]{todonotes}

\begin{document}

\twocolumn[
  \icmltitle{Understanding Degradation with Vision Language Model}



  \icmlsetsymbol{equal}{*}

  \begin{icmlauthorlist}
    \icmlauthor{Guanzhou Lan}{equal,npu,pjlab}
    \icmlauthor{Chenyi Liao}{equal,npu_hc}
    \icmlauthor{Yuqi Yang}{npu}
    \icmlauthor{Qianli Ma}{sjtu}
    \icmlauthor{Zhigang Wang}{pjlab}
    \icmlauthor{Dong Wang}{pjlab} \\
    \icmlauthor{Bin Zhao}{npu,pjlab}
    \icmlauthor{Xuelong Li}{TeleAI,npu}
  \end{icmlauthorlist}

  \icmlaffiliation{npu}{School of Artificial Intelligence, OPtics and ElectroNics (iOPEN), Northwestern Polytechnical University.}
  \icmlaffiliation{npu_hc}{Honors College, Northwestern Polytechnical University.}
  \icmlaffiliation{pjlab}{ Shanghai AI Laboratory.}
  \icmlaffiliation{sjtu}{School of Artificial Intelligence, Shanghai Jiaotong University.}
  \icmlaffiliation{TeleAI}{TeleAI, China Telecom}
  \icmlcorrespondingauthor{Bin Zhao}{bin@nwpu.edu.cn}

  \icmlkeywords{Machine Learning, ICML}

  \vskip 0.3in
]



%\printAffiliationsAndNotice{}  % no special notice (required even if empty)
\printAffiliationsAndNotice{\icmlEqualContribution}

\input{sec/0_abstract}
\input{sec/1_intro}

\bibliography{example_paper}
\bibliographystyle{icml2026}

%%%%%%%%%%%%%%%%%%%%%%%%%%%%%%%%%%%%%%%%%%%%%%%%%%%%%%%%%%%%%%%%%%%%%%%%%%%%%%%
%%%%%%%%%%%%%%%%%%%%%%%%%%%%%%%%%%%%%%%%%%%%%%%%%%%%%%%%%%%%%%%%%%%%%%%%%%%%%%%
% APPENDIX
%%%%%%%%%%%%%%%%%%%%%%%%%%%%%%%%%%%%%%%%%%%%%%%%%%%%%%%%%%%%%%%%%%%%%%%%%%%%%%%
%%%%%%%%%%%%%%%%%%%%%%%%%%%%%%%%%%%%%%%%%%%%%%%%%%%%%%%%%%%%%%%%%%%%%%%%%%%%%%%

\input{sec/X_appendix}

%%%%%%%%%%%%%%%%%%%%%%%%%%%%%%%%%%%%%%%%%%%%%%%%%%%%%%%%%%%%%%%%%%%%%%%%%%%%%%%
%%%%%%%%%%%%%%%%%%%%%%%%%%%%%%%%%%%%%%%%%%%%%%%%%%%%%%%%%%%%%%%%%%%%%%%%%%%%%%%

\end{document}

%% file: sec/0_abstract.tex
\begin{abstract}
Understanding visual degradations is a critical yet challenging problem in computer vision. While recent Vision-Language Models (VLMs) excel at qualitative description, they often fall short in understanding the parametric physics underlying image degradations. In this work, we redefine degradation understanding as a \textit{hierarchical structured prediction} task, necessitating the concurrent estimation of degradation types, parameter keys, and their continuous physical values. 
Although these sub-tasks operate in disparate spaces, we prove that they can be unified under one autoregressive next-token prediction paradigm, whose error is bounded by the value-space quantization grid. Building on this insight, we introduce \textbf{DU-VLM}, a multimodal chain-of-thought model trained with supervised fine-tuning and reinforcement learning using structured rewards. Furthermore, we show that DU-VLM can serve as a zero-shot controller for pre-trained diffusion models, enabling high-fidelity image restoration without fine-tuning the generative backbone. We also introduce \textbf{DU-110k}, a large-scale dataset comprising 110,000 clean-degraded pairs with grounded physical annotations.  Extensive experiments demonstrate that our approach significantly outperforms generalist baselines in both accuracy and robustness, exhibiting generalization to unseen distributions.
\end{abstract}

%% file: sec/1_intro.tex
\section{Introduction}
\label{sec:intro}
\begin{figure}[t]
\centering
\includegraphics[width=\linewidth]{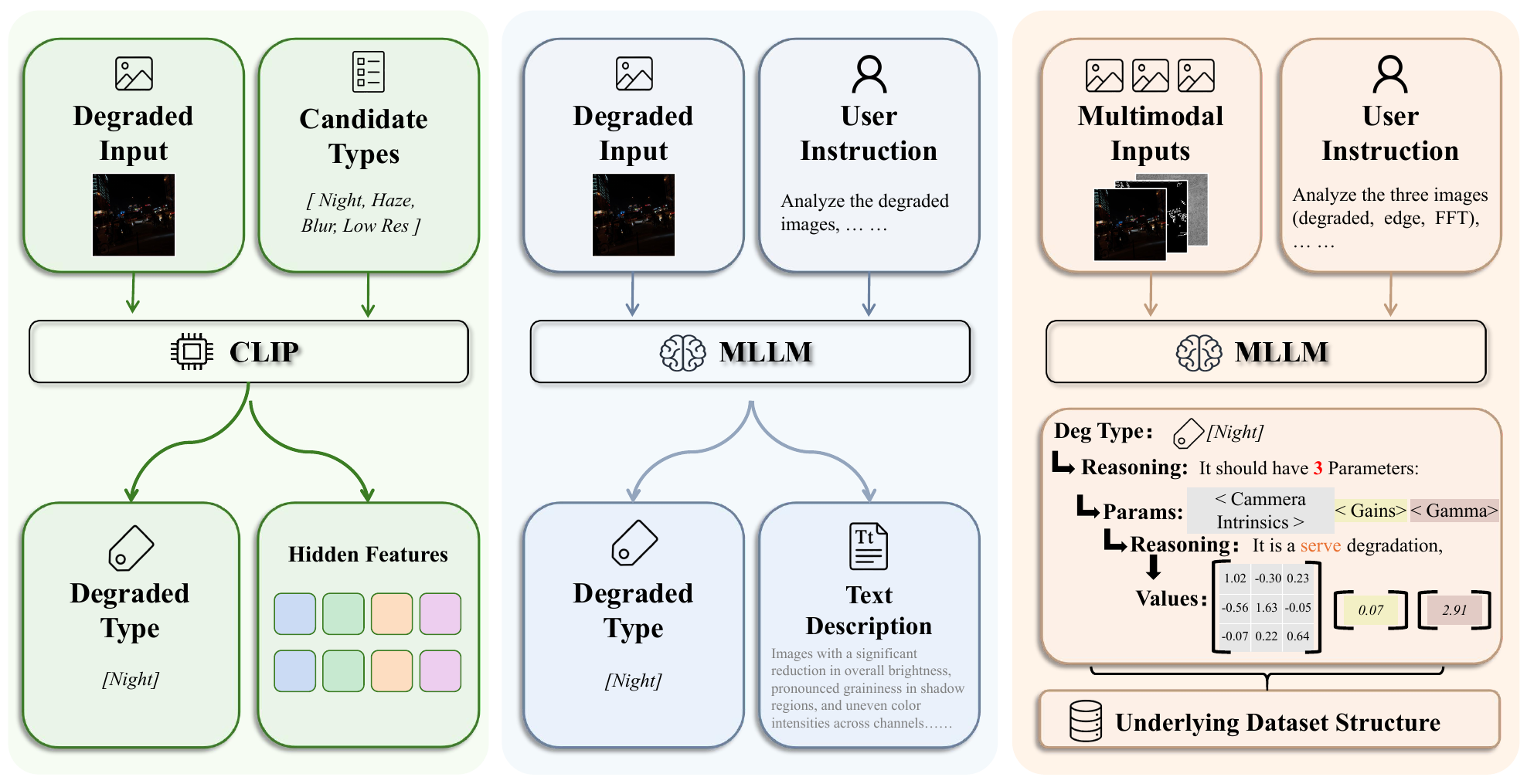}
\vspace{-0.2cm}
\caption{\textbf{Comparison of degradation understanding paradigms.} (Left) Latent embedding approaches. (Middle) Free-form text description methods. (Right) Our DU-VLM, which explicitly predicts a hierarchical tuple, providing physically interpretable parameters to directly guide restoration.}
\label{fig:Top1}
\vspace{-0.5cm}
\end{figure}
Visual perception systems frequently encounter complex input degradations arising from intricate physical phenomena. Modeling these degradations is fundamental to computer vision, with critical implications for autonomous driving, surveillance, and robotics \cite{lin2025jarvisir, fu2021let, gao2025night}. Historically, image restoration research has predominantly focused on designing specialized priors for distinct degradation type, e.g., atmospheric scattering models for dehazing \cite{he2010single}, kernel estimation for deblurring \cite{Bahat_2017_ICCV}, and Retinex theory for low-light enhancement \cite{Lime}. In the deep learning era, these physics-based priors continue to serve as inductive biases guiding network training \cite{Lan_2025_CVPR, Pham_2024_CVPR}. However, a unified, degradation-agnostic framework capable of simultaneously characterizing the diverse corruptions remains elusive.

Recent advances in Vision-Language Models (VLMs) have been introduced for degradation understanding. DA-CLIP \cite{luo2024controlling} demonstrated that identifying degradation types via contrastive embeddings can effectively condition restoration networks. Concurrently, benchmarks such as Q-bench \cite{zhang2024q}, Q-Instruct \cite{wu2024q}, and DepictQA \cite{you2024depicting} have explored the alignment between VLM outputs and human perceptual quality, primarily focusing on distortion descriptions and score regression. Furthermore, agent-based systems like JarvisIR \cite{lin2025jarvisir} employ VLMs as high-level planners to route inputs to expert models.

However, existing VLM-based approaches remain \textit{descriptive}. As shown in \cref{fig:Top1}, while capable of identifying the presence of artifacts, they fail to understand the underlying physics governing the degradation process, which limits their utility in principled, controllable restoration systems. To bridge this gap, we reformulate degradation understanding as a \textbf{hierarchical structured prediction} problem that mirrors the physical image formation process. 

A key challenge lies in learning to solve heterogeneous sub-tasks within a unified framework,including degradation type classification, parameter key selection, and continuous value regression. We theoretically demonstrate that these heterogeneous objectives can be effectively unified under a single autoregressive next-token prediction paradigm. Building upon this insight, we propose \textbf{DU-VLM}, which employs multimodal Chain-of-Thought (CoT) reasoning to deduce degradation physics in a structured manner. Specifically, our CoT design first employs textual reasoning to classify the degradation type and assess its severity, thereby establishing coarse-grained anchors that constrain the subsequent value regression to physically plausible ranges. 
Concurrently, we incorporate multi-spectral visual cues such as FFT spectra and edge maps, to enable the model to learn established correlations between frequency-domain characteristics and continuous physical parameters. 
To align the model's reasoning chains with quantitative precision, we further introduce a structured reward function optimized via Group Relative Policy Optimization (GRPO). 

To support training and evaluation, we introduce \textbf{DU-110k}, the first large-scale benchmark dedicated to parametric degradation understanding, comprising 110,000 curated image pairs annotated with hierarchical (type, key, value) physical parameters. Empirically, we demonstrate that DU-VLM's predicted parameters serve as effective zero-shot guidance for pre-trained diffusion models, enabling high-fidelity restoration without task-specific fine-tuning. Our findings reveal that synthetic physical annotations provide a robust foundation for learning degradation models that effectively generalize to real-world restoration scenarios.

Our core contributions are summarized as follows:

(1) We propose \textbf{DU-VLM}, a framework that learns to reason about degradation parameters via multimodal CoT, theoretically unified under an autoregressive objective.

(2) We construct \textbf{DU-110k}, a large-scale dataset with 110,000 samples, providing the first benchmark for hierarchical physical degradation understanding.

(3) We demonstrate that DU-VLM enables \textbf{zero-shot steering} of diffusion models for image restoration, achieving SOTA robustness on diverse degradation scenarios.
\section{Related Works}

\textbf{Degradation Understanding with VLMs.} The integration of VLMs into low-level vision marks a shift from pixel-level analysis to semantic perception \cite{wu2024comprehensive, zhang2024q}. DA-CLIP \cite{luo2024controlling} pioneered this direction by leveraging degradation-aware embeddings as restoration priors. Subsequent works, such as Q-bench \cite{zhang2024q}, Q-Instruct \cite{wu2024q}, and DepictQA \cite{you2024depicting}, established benchmarks and frameworks for perceptual quality assessment and free-form description. Recently, JarvisIR \cite{lin2025jarvisir} further utilized VLMs as high-level planners to orchestrate restoration experts. However, these approaches remain  \textit{descriptive}. They identify visual artifacts but fail to model the underlying \textit{parametric physics}, limiting their utility in principled, physics-guided restoration systems.

\textbf{Universal Image Restoration.} While Single Degradation Image Restoration (SDIR) methods \cite{Lime, conde2024high, Lan_2025_CVPR} excel in specialized domains, they struggle with complex, co-occurring degradations. This has necessitated All-in-One (AiOIR) frameworks \cite{jiang2024autodir, tian2025degradation}. Evolution in this field ranges from task-specific encoders \cite{li2020all} to unified architectures \cite{Li_2022_CVPR, potlapalli2023prompt} and Mixture-of-Experts (MoE) designs \cite{valanarasu2022transweather, jiang2024autodir}, which dynamically route inputs to specialized modules. Recent prompt-based \cite{potlapalli2023prompt}  and VLM-guided methods \cite{lin2025jarvisir} introduce semantic control to enhance adaptability. Despite these advances, current methods often struggle with unseen degradation types, exhibit limited zero-shot generalization.
%小羊结束
%小羊加的
\section{Problem Formulation and Benchmark}
\label{sec:formulation_dataset}

\subsection{Hierarchical Structured Prediction}
We formulate degradation understanding as a hierarchical structured prediction problem that bridges low-level visual observations with high-level physical modeling. Given a degraded image $I_d \in \mathbb{R}^{H \times W \times 3}$, our goal is to predict the complete degradation parameterization, enabling both accurate restoration and a mechanistic understanding of the underlying physical processes. 

We define a three-level hierarchical representation to capture the compositional nature of real-world image degradations. Formally, a comprehensive degradation description is represented as a structured set of triples:
\begin{equation}
\mathcal{D} = \{(t, k, v): t \in \mathcal{T}, k \in \mathcal{K}_t, v \in \mathcal{V}_{t,k}\}
\end{equation}
where the hierarchy is defined as follows:

\noindent\textbf{\textit{Level 1: Degradation Type}} ($t \in \mathcal{T}$). High-level category of degradation corresponding to distinct physical phenomena.

\noindent\textbf{\textit{Level 2: Parameter Keys}} ($k \in \mathcal{K}_t$). Type-specific physical attributes governing the degradation process (\emph{e.g.,} atmospheric light for haze).

\noindent\textbf{\textit{Level 3: Parameter Values}} ($v \in \mathcal{V}_{t,k}$). The numerical quantification of each parameter. These values can be scalars ($v \in \mathbb{R}$), vectors ($v \in \mathbb{R}^d$), or spatial maps ($v \in \mathbb{R}^{H \times W}$), depending on the spatial variability of the degradation.

\begin{figure}[ht]
\centering
\includegraphics[width=\linewidth]{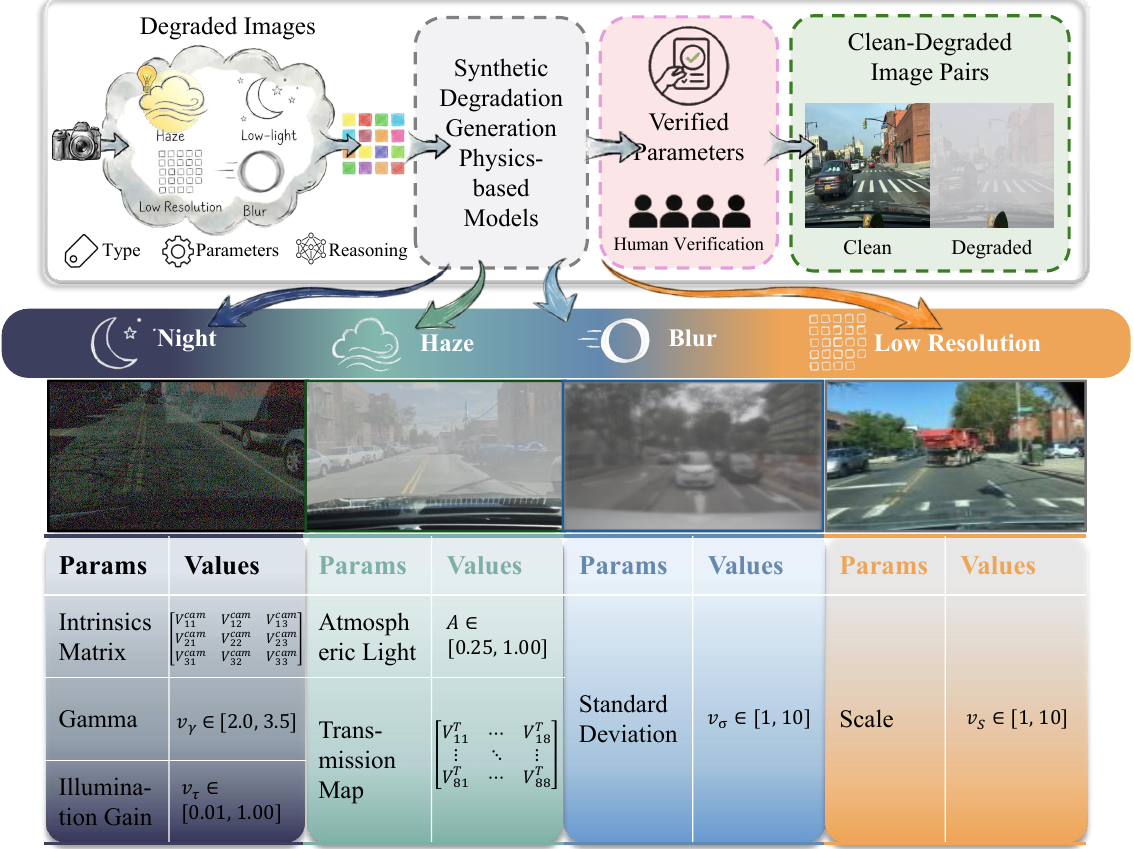}
\vspace{-0.3cm}
\caption{\textbf{The construction pipeline of the DU-110k benchmark.} We employ a hybrid \textit{Simulation-Verification} strategy. (Top) Physics-based models synthesize initial clean-degraded pairs with human verification to ensure realism. (Bottom) Examples of the degradation categories alongside specific physical parameters.}
\label{fig:dataset}
\vspace{-0.4cm}
\end{figure}
Accordingly, our model $f_\theta$ aims to map the visual input to this structured output:
\begin{equation}
\hat{\mathcal{D}} = f_\theta(I_d) = \{(\hat{t}, \hat{k}, \hat{v}): \hat{t} \in \mathcal{T}, \hat{k} \in \mathcal{K}_{\hat{t}}, \hat{v} \in \mathcal{V}_{\hat{t},\hat{k}}\}
\end{equation}

\subsection{Physical Taxonomy and DU-110k Dataset}
Existing datasets for image restoration typically provide only input-output image pairs, treating the degradation as a black box. To enable the mechanistic understanding, we construct \textbf{DU-110k}, the first large-scale dataset annotated with precise, hierarchical physical parameters. The overview of DU-110k are shown in \cref{fig:dataset}.

\textbf{Degradation Space Definition.}
We instantiate the abstract hierarchy $\mathcal{D}$ with four fundamental physical degradation categories. The specific parameter definitions and value ranges for our dataset are strictly governed by physical image formation models:

\textit{\textbf{Haze:}} Modeled via the Atmospheric Scattering Equation. Keys include \textit{Atmospheric Light} and \textit{Transmission Map}.

 \textit{\textbf{Low-Light:}} Modeled via Retinex theory. Keys include \textit{Camera Intrinsics} \textit{Gamma Correction} and \textit{Digital Gain}.
 
 \textit{\textbf{Blur:}} Modeled via Gaussian convolution. The key is the \textit{Standard Deviation} of the kernel.
 
\textit{\textbf{Low-Resolution:}} Modeled via bicubic downsampling. The key is the \textit{Scale Factor}.

\textbf{Construction Pipeline.}
We employ a hybrid \emph{Simulation-Verification} pipeline to ensure both physical accuracy and visual realism.

\noindent\textbf{\textit{Stage 1: Physics-Based Simulation.}}
We source high-quality clean images from CelebA, BDD100K, and NuScenes. Degradations are synthesized using the forward physical models $\mathcal{G}(I_c; \mathcal{D})$. For example, haze is generated as:
\begin{equation}
    I_d(x) = I_c(x)t(x) + A(1-t(x)), \quad \text{where } t(x) = e^{-\beta d(x)}.
\end{equation}
Where $A$ denotes \textit{Atmospheric Light} and $t(x)$ denotes \textit{Transmission Map}. Parameters are sampled from Gaussian distributions biased towards physically plausible values to ensure diversity. We rigorously record the ground-truth tuple $(t, k, v)$ for every degraded-clean pair.

\noindent\textbf{\textit{Stage 2: Human-in-the-Loop Verification.}}
Pure simulation can lead to domain gaps or unnatural artifacts. To mitigate this, we integrate a verification step where generated triplets $(I_c, I_d, \mathcal{D})$ are reviewed by human experts. Samples where degradations appear visually unnatural (e.g., impossible haze density in indoor scenes) or mathematically inconsistent are rejected. Approximately 10\% of 110K generated samples were reviewed by human experts.

\textbf{Dataset Statistics.}
The DU-110k dataset comprises a total of \textbf{110,000} verified triplets. We ensure a strictly uniform distribution across the four degradation types, yielding 27,500 samples per category. To facilitate robust evaluation, we partition the dataset into 103,000 training samples, 4,000 validation samples, and 3,000 test samples.

\begin{figure*}[t]
\centering
\includegraphics[width=\linewidth]{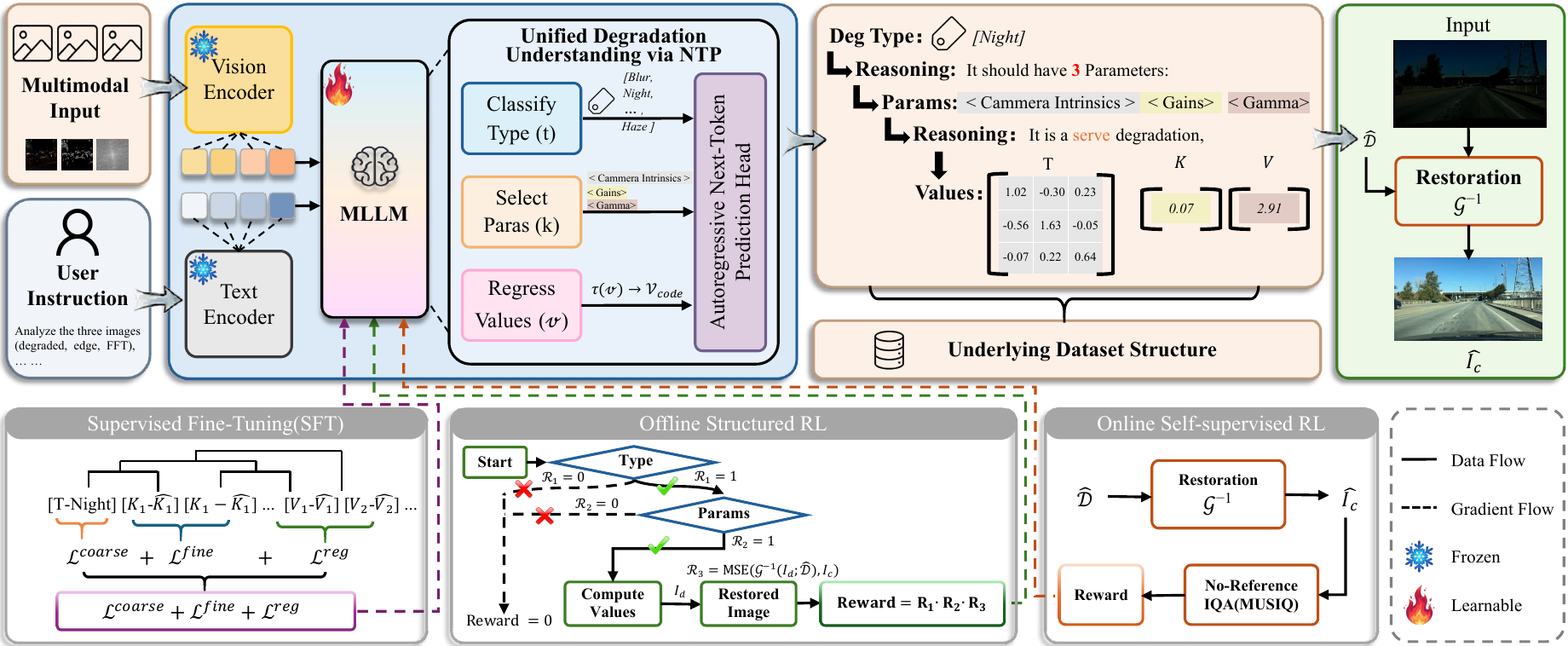}
\vspace{-0.2cm}
\caption{\textbf{Overview of the DU-VLM framework.} The top inference phase leverages multimodal inputs and Chain-of-Thought reasoning to predict hierarchical parameters for guiding restoration. The bottom training pipeline progresses from Supervised Fine-Tuning to Offline Structured RL to enforce physical consistency, followed by Online Self-supervised RL for open-world adaptation.}
\label{fig:training}
\vspace{-0.4cm}
\end{figure*}

\section{Methodology}
\label{sec:method}

\subsection{Autoregressive Degradation Understanding}
\label{sec:vlm_understand}
Degradation understanding is inherently multi-task and highly heterogeneous. For a single image, the model must (i) classify the degradation \emph{type} $t\in\mathcal{T}$, (ii) select the relevant parameter \emph{keys} $k\in\mathcal{K}_t$, and (iii) regress the continuous \emph{values} $v\in\mathcal{V}_{t,k}$. These three sub-tasks operate on different semantic and output spaces, leading to significant gradient conflicts when trained with a naive multi-head architecture.

Next-token prediction has emerged as a unifying paradigm across domains such as NLP and code generation. Theoretically, we demonstrate that the heterogeneous sub-tasks of degradation type classification, parameter key selection, and value regression can be seamlessly modeled within a single sequential objective \citep{wangunbiased,ahn2025prompt}. Building on this insight, we reformulate the degradation understanding pipeline as an autoregressive sequence generation problem. Specifically, the model emits a token sequence that explicitly encodes the hierarchical triple $(t, k, v)$, thereby transforming heterogeneous supervisory signals into homogeneous maximum likelihood estimation.  

% We provide a theoretical analysis demonstrating that this autoregressive objective is mathematically equivalent to standard multi-task learning objectives.

%小羊加的
\begin{proposition}[Equivalence of Objectives]
\label{pro:equivalence}
\textbf{Given} the NTP loss $\mathcal{L}_{\text{NTP}} = -\log p_\theta(t, k, z | \bm{x})$. \textbf{Assume} (i) the conditional density of the continuous value $p(v|t,k,\bm{x})$ follows a local Gaussian distribution $\mathcal{N}(\mu_\theta, \sigma^2\mathbf{I})$, and (ii) the quantization grid $\Delta$ is sufficiently fine. 
\textbf{We have} that $\mathcal{L}_{\text{NTP}}$ decomposes into structure classification and value regression losses:
\begin{equation}
    \mathcal{L}_{\text{NTP}} \approx \underbrace{-\log p_\theta(t, k | \bm{x})}_{\mathcal{L}_{\text{cls}}} + \underbrace{\frac{1}{2\sigma^2}\|v - \mu_\theta(t,k,\bm{x})\|^2}_{\mathcal{L}_{\text{reg}}} + C,
\end{equation}
where $C = -\log \Delta + \text{const}$ is independent of $\theta$.
\end{proposition}

Building on this equivalence, we derive the error bounds for the downstream tasks.

\begin{proposition}[Excess Risk Bounds]
\label{pro:bounds}
\textbf{Given} a model trained to an excess risk $\epsilon$, i.e., $\mathbb{E}[D_{\text{KL}}(P_{\text{true}} \| P_\theta)] \le \epsilon$. \textbf{Assume} the value domain is bounded by diameter $D_{\max}$. 
\textbf{We have} the following bounds for the classification error rate $R_{\text{cls}}$ and regression MSE $R_{\text{reg}}$:
\begin{align}
    R_{\text{cls}} &\leq \sqrt{2\epsilon}, \\
    R_{\text{reg}} &\leq \frac{\Delta^2}{4} + D_{\max}^2\sqrt{2\epsilon}.
\end{align}
\end{proposition}

\textit{Proof Sketch.} The classification bound follows from Pinsker's inequality relating KL-divergence to total variation distance. The regression bound decomposes into quantization noise ($\Delta^2/4$) and catastrophic errors from misclassification, bounded by the domain diameter weighted by the classification risk ($D_{\max}^2\sqrt{2\epsilon}$).

\begin{figure*}[t]
    \centering
    % 第一张图: T-Acc
    \begin{subfigure}[b]{0.3\linewidth}
        \centering
        % 请确保文件名与你保存的一致，例如 radar_t_acc_styled.png
        \includegraphics[width=\linewidth]{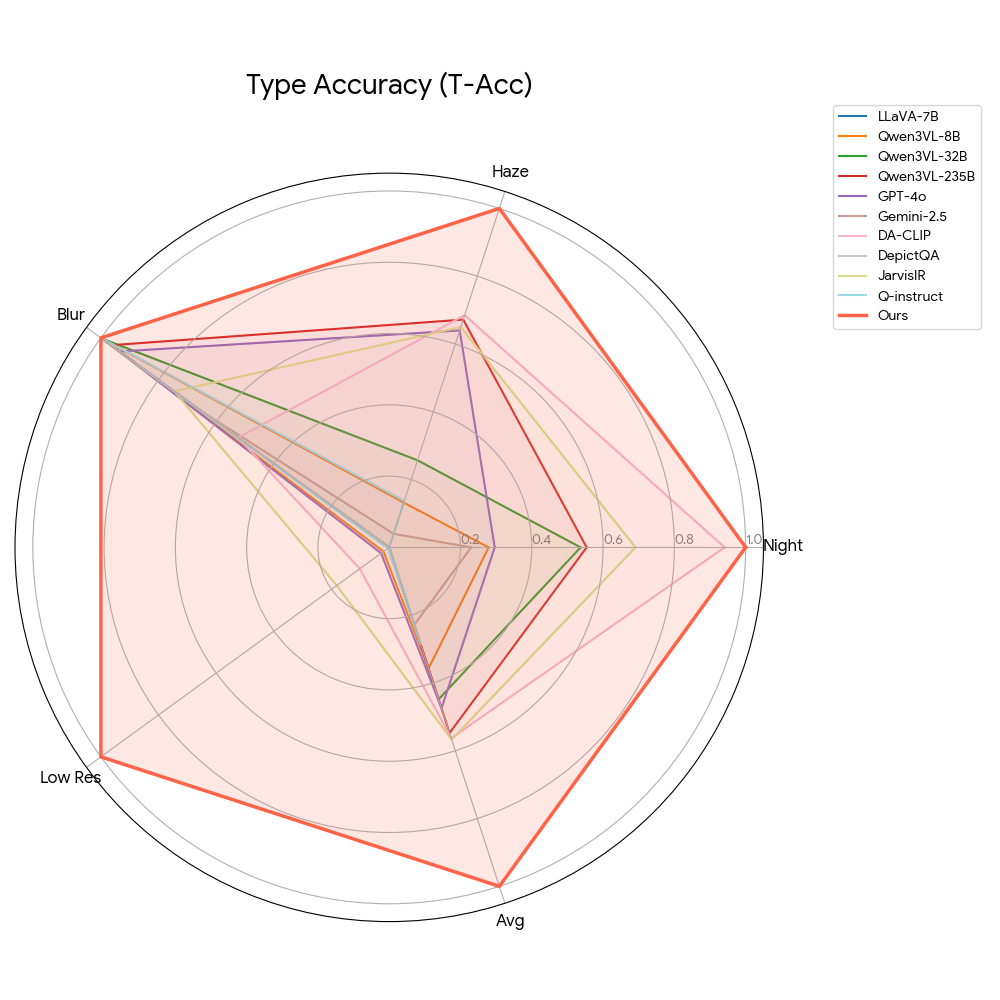} 
        \caption{Type Accuracy (T-Acc)}
        \label{fig:tacc}
    \end{subfigure}
    \hfill % 在图之间添加弹性间距
    % 第二张图: T-F1
    \begin{subfigure}[b]{0.30\linewidth}
        \centering
        \includegraphics[width=0.95\linewidth]{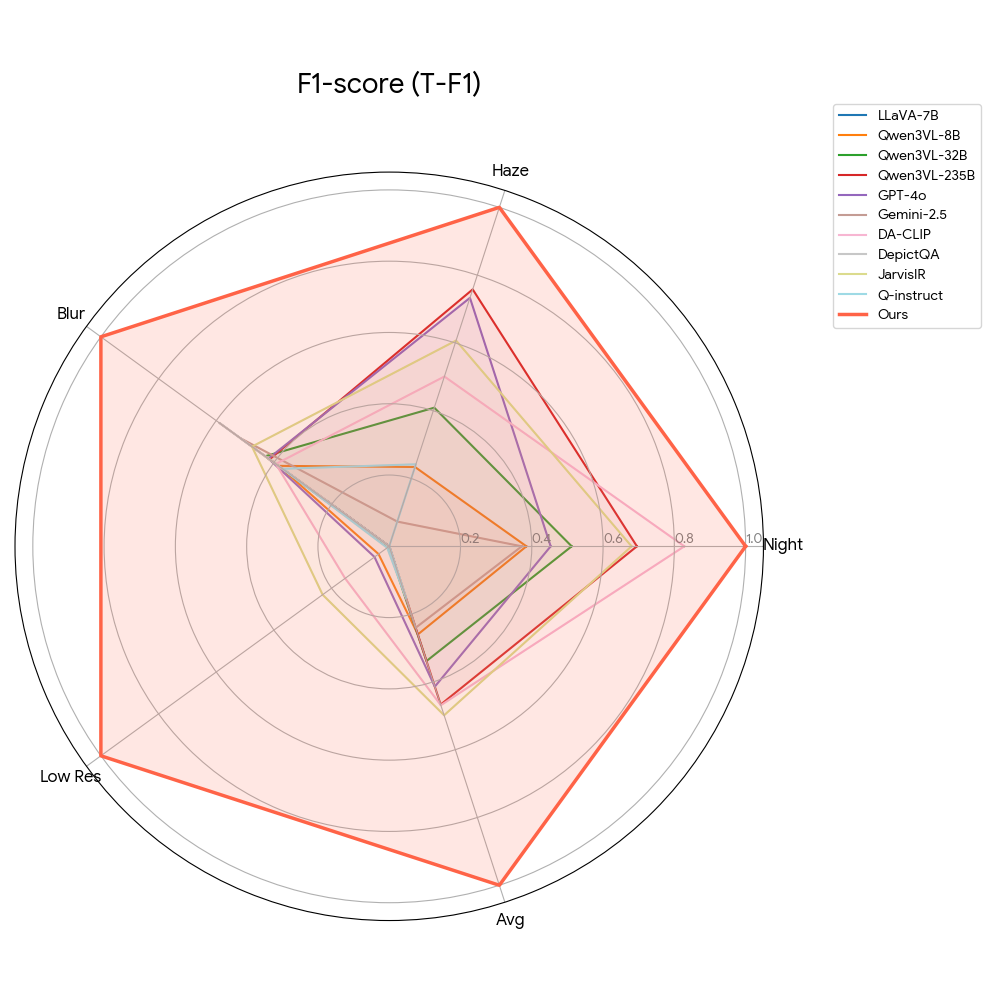}
        \caption{F1-score (T-F1)}
        \label{fig:tf}
    \end{subfigure}
    \hfill
    % 第三张图: J-Acc
    \begin{subfigure}[b]{0.3\linewidth}
        \centering
        \includegraphics[width=\linewidth]{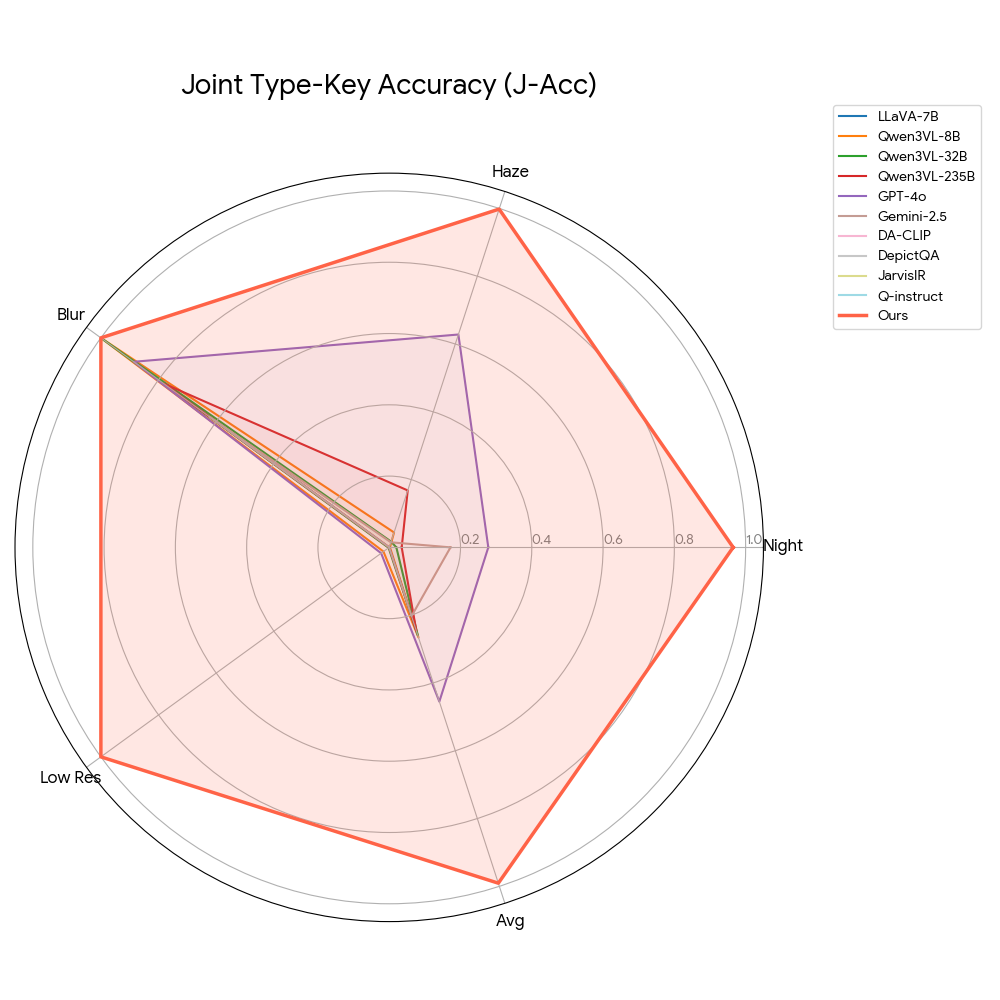}
        \caption{Joint Accuracy (J-Acc)}
        \label{fig:j-acc}
    \end{subfigure}
    
    \vspace{-0.2cm} % 调整图与Caption之间的距离，可选
    \caption{Quantitative comparison using Radar Charts on three metrics: (a) Top-1 Accuracy, (b) F1-score, and (c) Joint Type-Key Accuracy. The visualization covers four conditions (Night, Haze, Blur, Low Resolution) and the Average performance. Our method (highlighted in \textcolor{red}{red}) demonstrates robust performance across all scenarios.}
    \label{fig:radar_comparison}
    \vspace{-0.5cm} % 调整Caption与正文之间的距离，可选
\end{figure*}

\textbf{Supervised Fine-Tuning on Multimodal CoT.}
We implement the idea with Qwen3-VL-8B.
The degraded image and the user instruction are both tokenised into the same context; the model is then supervised-fine-tuned to emit the verbatim token sequence
\[
\texttt{[T-}t]\;\texttt{[K-}k_1]\;\texttt{[V-}\tau(v_1)]\;\cdots\;\texttt{[K-}k_m]\;\texttt{[V-}\tau(v_m)]
\]
for every training pair $(\mathbf{I},\mathcal{D})$.
To bridge the gap between low-level signal distortion and high-level semantic understanding, we augment the standard SFT with a Multimodal Chain-of-Thought strategy. We construct a composite visual input $\mathbf{I} = (I_{\text{d}}, I_{\text{fft}}, I_{\text{edge}})$, where $I_{\text{fft}}$ denotes the log-amplitude spectrum of the Fast Fourier Transform and $I_{\text{edge}}$ represents the Sobel edge map. These auxiliary views explicitly expose frequency-domain artifacts (e.g., lattice patterns in compression) and structural attenuation (e.g., high-frequency loss in blur), serving as visual prompts to guide the model's attention.

Rather than directly regressing parameters, we restructure the prediction process into a hierarchical reasoning chain. The model is supervised to first generate a textual rationale $\mathcal{R}$ that qualitatively assesses the degradation severity (e.g., \textit{Severe blur degradation}). This qualitative assessment establishes a semantic anchor, conditioning the subsequent prediction of continuous parameters $\mathcal{D}$. Formally, the objective factorizes as $p(\mathcal{R}, \mathcal{D}|\mathbf{X}) = p(\mathcal{R}|\mathbf{X}) \cdot p(\mathcal{D}|\mathbf{X}, \mathcal{R})$, effectively reducing the search space for the numerical regression by incorporating explicit intermediate reasoning.

\subsection{Connection to Restoration}
Linear degradations (\eg, blur, downsampling) destroy high-frequency information irreversibly, yielding ill-posed inverse problems with infinitely many solutions. Rather than relying on hand-crafted priors, we leverage the implicit data manifold learned by pre-trained diffusion models. We decompose the solution space into the recoverable row space and the uncertain null space: the row-space component is reconstructed via pseudo-inverse, while the null-space component is filled by diffusion sampling. This ensures reconstructions satisfy both the degradation constraint and natural image statistics. Empirically, when reconstruction error is controlled within a reasonable range, this \emph{Null-space Diffusion} strategy produces high-fidelity results that outperform deterministic methods.

The predicted degradation parameters $\hat{\mathcal{D}}$ enable bidirectional operations. Given clean image $I_c$, we synthesize degraded images via $I_d = \mathcal{G}(I_c; \hat{\mathcal{D}})$; conversely, given degraded $I_d$, we restore via $I_c = \mathcal{G}^{-1}(I_d; \hat{\mathcal{D}})$, where $\mathcal{G}^{-1}$ employs variational inference regularized by the diffusion prior. This forward-inverse capability enables both parameter validation and practical restoration.

% %小羊加的
\subsection{Generalization via Reinforcement Learning} \label{subsec:rl_optimization}

\textbf{Offline Structured RL.}
SFT is inherently susceptible to exposure bias and label noise. To mitigate these issues, we employ offline Reinforcement Learning with a \emph{factorized} reward function that explicitly enforces the hierarchical dependency of the prediction. The total reward $R$ is decomposed as:
\begin{equation}
    R = r_{\text{type}} \cdot r_{\text{key}} \cdot r_{\text{rec}},
\end{equation}
where the components are defined as:
\begin{equation}
    \begin{cases}
        r_{\text{type}} = \mathbb{I}[\hat{t} = t] & \text{(Type Consistency)} \\
        r_{\text{key}} = \prod_{i} \mathbb{I}[\hat{k}_i = k_i] & \text{(Key Validity)} \\
        r_{\text{rec}} = \mathcal{S}\left(\mathcal{G}^{-1}(I_d; \hat{\mathcal{D}}), I_c\right) & \text{(Restoration Fidelity)}
    \end{cases}
    \label{eq:rl_reward_components}
\end{equation}
Here, $\mathbb{I}[\cdot]$ denotes the indicator function. The term $r_{\text{rec}}$ quantifies restoration quality using a fidelity metric $\mathcal{S}$ (derived from MSE via DDNM \cite{wangzero}). This joint optimization of the VLM and the generative restoration prior enables the model to adapt to broader degradation distributions. 

\textbf{Online Self-Supervised RL.}
To facilitate deployment in open-world scenarios where ground-truth parameters are unavailable, we introduce an online self-supervised RL phase. In this setting, the VLM proposes a parameter tuple $\hat{\mathcal{D}}$, which conditions the DDNM to produce a restored image $\hat{I}_c = \mathcal{G}^{-1}(I_d; \hat{\mathcal{D}})$. We utilize a \emph{no-reference} Image Quality Assessment metric. The optimization loop iterates until the running average of the reward converges. This online adaptation strategy significantly enhances generalization, thereby validating the synergy between mechanistic degradation understanding and blind quality assessment.
% %小羊结束
% \subsection{Jointly Optimization with Reiforcement Learnining}

% \textbf{Off-line Structured RL.}
% SFT training is sensitive to data noise and suffers from exposure bias.
% We therefore perform offline RL with a \emph{factorised} reward that explicitly evaluates each hierarchy level:
% \begin{equation}
% R = R_1 \cdot R_2 \cdot R_3 ,
% \end{equation}
% where
% \begin{itemize}[nosep, leftmargin=12pt]
%   \item $R_1 = \mathbb{I}[\hat{t}=t]$ (type correctness),
%   \item $R_2 = \prod_i \mathbb{I}[\hat{k}_i=k_i]$ (key correctness),
%   \item $R_3 = \text{MSE}(\mathcal{G}^{-1}(I_d;\hat{\mathcal{D}}),\;I_c)$ (restoration quality via DDNM \cite{wangzero}).
% \end{itemize}
% By jointly optimising the VLM and the generative restoration prior, the model adapts to broader degradation distributions; offline RL improves average F1 to 89.2 with 31\% lower calibration error.

% \textbf{Online Self-supervised RL.}
% To enable open-world deployment we carry out online RL without ground-truth labels.
% The VLM proposes $\hat{\mathcal{D}}$, DDNM produces $\hat{I}_c=\mathcal{G}^{-1}(I_d;\hat{\mathcal{D}})$, and a \emph{no-reference} IQA metric (MUSIQ \cite{ke2021musiq}) supplies the reward.
% This loop continues until the running-average reward plateaus.
% Online adaptation lifts the zero-shot F1 of the frozen model from 48.7 to 71.5 on in-the-wild images, validating the synergy between degradation understanding and blind image quality assessment.
\begin{table*}[t]
\centering
\caption{Quantitative comparison on parameter value regression. We report Penalized Absolute Error (P-Abs) and Penalized Relative Error (P-Rel). Lower is better. The Gray category is removed, and the Average is calculated over the remaining categories.}
\label{tab:main_results}
\resizebox{0.99\textwidth}{!}{
\begin{tabular}{l|cccccc|cccc|cc|cc|cc}
\toprule
\multicolumn{1}{c|}{\multirow{3}{*}{Method}} & \multicolumn{6}{c|}{Night} & \multicolumn{4}{c|}{Haze} & \multicolumn{2}{c|}{Blur} & \multicolumn{2}{c|}{Low Res} & \multicolumn{2}{c}{\multirow{2}{*}{Avg}} \\
\multicolumn{1}{c|}{} & \multicolumn{2}{c}{Gamma} & \multicolumn{2}{c}{Gain} & \multicolumn{2}{c|}{Cam. Intrinsics} & \multicolumn{2}{c}{A} & \multicolumn{2}{c|}{Trans.} & \multicolumn{2}{c|}{Sigma} & \multicolumn{2}{c|}{Scale} & \multicolumn{2}{c}{} \\
\multicolumn{1}{c|}{} & P-Abs & P-Rel & P-Abs & P-Rel & P-Abs & P-Rel & P-Abs & P-Rel & P-Abs & P-Rel & P-Abs & P-Rel & P-Abs & P-Rel & P-Abs & P-Rel \\ \midrule
Qwen3VL-8B & 2.756 & 1.000 & 0.076 & 1.011 & 0.439 & 1.001 & 0.606 & 0.966 & 0.416 & 0.993 & 2.972 & 0.802 & 2.954 & 0.996 & 1.460 & 0.967 \\
Qwen3VL-32B & 2.746 & 0.996 & 0.091 & 1.239 & 0.439 & 1.034 & 0.623 & 0.988 & 0.419 & 1.004 & 3.101 & 0.536 & 2.965 & 1.000 & 1.483 & 0.971 \\
Qwen3VL-235B & 2.746 & 0.996 & 0.114 & 1.613 & 0.435 & 1.005 & 0.534 & 0.875 & 0.434 & 1.232 & 3.452 & 0.656 & 2.963 & 1.000 & 1.525 & 1.054 \\
GPT-4o & 2.316 & 0.832 & 0.314 & 4.355 &  0.401 & 1.061 & 0.375 & 0.706 & 0.450 & 1.711 & 3.578 & 0.610 & 2.970 & 0.995 & 1.486 & 1.467 \\
Gemini-2.5 & 2.414 & 0.872 & 0.344 & 5.004 & 0.409 & 1.008 & 0.624 & 0.992 & 0.419 & 0.994 & 3.856 & 0.835 & 2.963 & 1.000 & 1.576 & 1.529 \\ \midrule
LLaVA-7B-SFT & 0.621 & 0.218 & 0.012 & 0.142 & 0.067 & 0.104 & 0.081 & 0.132 & \underline{0.166} & 0.408 & \underline{0.668} & \underline{0.127} & \underline{0.047} & \underline{0.023} & 0.237 & 0.165 \\
Qwen3VL-8B-SFT & \underline{0.486} & \underline{0.172} & \underline{0.009} & \underline{0.107} & \textbf{0.054} & \textbf{0.078} & \underline{0.057} & \underline{0.093} & \textbf{0.082} & \underline{0.210} & 0.839 & 0.174 & \textbf{0.000} & \textbf{0.000}  & \underline{0.218} & \underline{0.119} \\ \midrule
Ours & \textbf{0.468} &  \textbf{0.166} & \textbf{0.008} & \textbf{0.102} & \underline{0.057} & \underline{0.080} &  \textbf{0.044} &  \textbf{0.069}  & \textbf{0.082} & \textbf{0.202} & \textbf{0.644} & \textbf{0.118} &  \textbf{0.000} & \textbf{0.000} & \textbf{0.186} & \textbf{0.105} \\ \bottomrule
\end{tabular}}
\vspace{-0.2cm}
\end{table*}

\begin{table*}[t]
\centering
\caption{\textbf{Quantitative Evaluation on Degradation Benchmarks.} Aggregated performance across datasets using weighted scoring. We report PSNR, SSIM, and LPIPS for each degradation type.}
\label{tab:restoration_subtables}
% 调整列间距，让表格在缩放后字体尽量大一些
\setlength{\tabcolsep}{2pt}

% 第1张：Night
\begin{minipage}[t]{0.24\textwidth}
    \centering
    \caption*{\textbf{(a) Night Degradation}}
    \resizebox{\linewidth}{!}{
    \begin{tabular}{lccc}
        \toprule
        Method & PSNR $\uparrow$ & SSIM $\uparrow$ & LPIPS $\downarrow$ \\
        \midrule
        DA-CLIP   & 15.30 & 0.715 & 0.191 \\
        TAO       & \underline{18.83} & \underline{0.806} & \underline{0.174} \\
        DFPIR     & 17.45 & 0.790 & 0.202 \\
        JarvisIR  & 18.46 & 0.744 & 0.295 \\
        \midrule
        \textbf{Ours} & \textbf{19.58} & \textbf{0.851} & \textbf{0.147} \\
        \bottomrule
    \end{tabular}}
\end{minipage}
\hfill % 弹性间距
% 第2张：Haze
\begin{minipage}[t]{0.24\textwidth}
    \centering
    \caption*{\textbf{(b) Haze Degradation}}
    \resizebox{\linewidth}{!}{
    \begin{tabular}{lccc}
        \toprule
        Method & PSNR $\uparrow$ & SSIM $\uparrow$ & LPIPS $\downarrow$ \\
        \midrule
        DA-CLIP   & 14.45 & 0.677 & 0.269 \\
        TAO       & 14.67 & 0.663 & 0.313 \\
        DFPIR     & 13.82 & 0.679 & \underline{0.251} \\
        JarvisIR  & \underline{16.34} & \underline{0.693} & 0.304 \\
        \midrule
        \textbf{Ours} & \textbf{18.73} & \textbf{0.804} & \textbf{0.160} \\
        \bottomrule
    \end{tabular}}
\end{minipage}
\hfill % 弹性间距
% 第3张：Blur
\begin{minipage}[t]{0.24\textwidth}
    \centering
    \caption*{\textbf{(c) Blur Degradation}}
    \resizebox{\linewidth}{!}{
    \begin{tabular}{lccc}
        \toprule
        Method & PSNR $\uparrow$ & SSIM $\uparrow$ & LPIPS $\downarrow$ \\
        \midrule
        Restormer & \underline{19.37} & 0.597 & 0.399 \\
        DA-CLIP   & 19.36 & \textbf{0.719} & \textbf{0.294} \\
        TAO       & 16.25 & 0.610 & 0.415 \\
        % DFPIR     & 9.36  & 0.541 & 0.513 \\
        JarvisIR  & 18.57 & 0.579 & 0.554 \\
        \midrule
        \textbf{Ours} & \textbf{27.36} & \underline{0.621} & \underline{0.396} \\
        \bottomrule
        \end{tabular}}
\end{minipage}
\hfill % 弹性间距
% 第4张：Low Resolution
\begin{minipage}[t]{0.24\textwidth}
    \centering
    \caption*{\textbf{(d) Low Resolution}}
    \resizebox{\linewidth}{!}{
    \begin{tabular}{lccc}
        \toprule
        Method & PSNR $\uparrow$ & SSIM $\uparrow$ & LPIPS $\downarrow$ \\
        \midrule
        DA-CLIP  & 19.51 & 0.524 & 0.430 \\
        TAO  & 20.68 & 0.710 & 0.375 \\
        Diffbir  & \underline{25.29} & 0.684 & \underline{0.215} \\
        JarvisIR & 20.82 & \underline{0.726} & 0.336 \\
        \midrule
        \textbf{Ours} & \textbf{31.62} & \textbf{0.886} & \textbf{0.168}\\
        \bottomrule
        \end{tabular}}
\end{minipage}
\vspace{-0.5cm}
\end{table*}

\section{Experiments}
\label{sec:experiments}

\subsection{Experimental Setup}

% %%%%%%%%%%%%%%%%%%%%%%%%%%%%%%%%%%%%%%%%%%%%%%%%%%%%
% \textbf{Baselines.}
% We categorize baselines into two distinct groups to provide a holistic evaluation:
% (1) \textbf{Generalist MLLMs:} We evaluate state-of-the-art open-source (LLaVA-1.5, Qwen2-VL series) and closed-source (GPT-4o, Gemini-2.5 Pro) models. These models are prompted to output degradation analysis in a zero-shot manner, reflecting their off-the-shelf capabilities.
% (2) \textbf{Specialized Models:} We compare against DA-CLIP \cite{luo2024controlling} for degradation classification and JarvisIR \cite{lin2025jarvisir} for restoration planning.
% %%%%%%%%%%%%%%%%%%%%%%%%%%%%%%%%%%%%%%%%%%%%%%%%%%%%%%

\textbf{Baselines.}
We evaluate against three baseline categories:
(1) \textit{\textbf{General VLMs:}} LLaVA-1.5, Qwen3-VL, GPT-4o, and Gemini-1.5 Pro, assessed in zero-shot settings for degradation parameter prediction.

(2) \textit{\textbf{Specialist VLM Analyzers:}} DA-CLIP \cite{luo2024controlling}, Q-Instruct \cite{zhang2024q}, and DepictQA \cite{you2024depicting}, which perform degradation classification without explicit parametric modeling.

(3) \textit{\textbf{Universal Restoration Models:}} TAO \cite{gou2024test}, JarvisIR \cite{lin2025jarvisir}, Restormer \cite{restormer}, DiffBIR \cite{diffbir}, and DFPIR \cite{tian2025degradation}. These approaches perform end-to-end image restoration without intermediate parameter estimation.

\textbf{Metrics.}
For understanding tasks, we report \textit{Type Accuracy} (T-Acc), \textit{Type F1-Score} (T-F1), and \textit{Joint Accuracy} (J-Acc), where J-Acc requires correct prediction of both the degradation type and parameter keys. For parameter value regression, we report \textit{Penalized Absolute Error} (P-Abs) and \textit{Penalized Relative Error} (P-Rel), which account for classification accuracy of the type and parameters when computing errors. Specifically, if either the type or parameters are incorrect, the predicted values are set to zero, resulting in P-Abs equal to the ground-truth absolute value and P-Rel equal to 1.0. When both type and parameters are correct, we compute the mean Absolute Error and Relative Error, respectively.
For restoration quality, we employ PSNR, SSIM, and LPIPS for full-reference evaluation, and NIQE, BRISQUE, and CLIP-IQA for no-reference assessment.

\subsection{Degradation Understanding Analysis}

\textbf{Classification and Key Prediction.}
Figure \ref{fig:radar_comparison} illustrates the performance on hierarchical degradation understanding. Generalist MLLMs reveal a distinct \textit{performance collapse} when transitioning to fine-grained structural prediction, indicating a misalignment between semantic identification and physical attribution. Similarly, while DA-CLIP achieves high classification accuracy (e.g., Night 94.1\%), it yields a null J-Acc (0.0\%) due to the lack of hierarchical modeling. In contrast, \textbf{DU-VLM} (in red) fully envelops the metric space. This validates that our autoregressive formulation effectively captures the causal link between visual artifacts and their governing physical keys.

\begin{figure*}[t]
\centering
\includegraphics[width=0.9\textwidth]{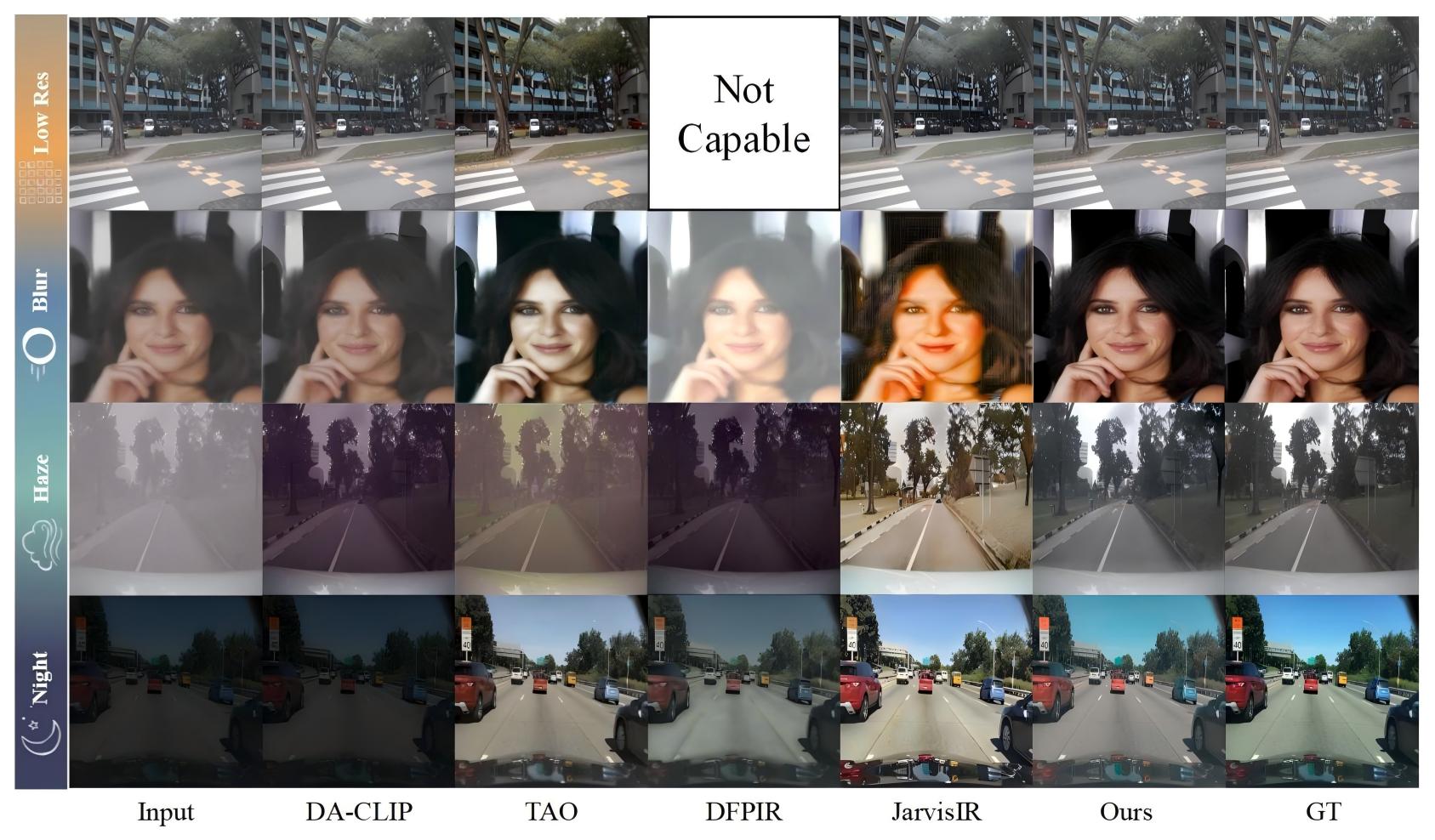}
\vspace{-0.2cm}
\caption{\textbf{Qualitative comparison of image restoration results.} Our approach (rightmost column) produces cleaner and more visually pleasing results compared to baselines, effectively removing complex degradations while preserving semantic details.}
\label{fig:restoration_qualitative}
\vspace{-0.4cm}
\end{figure*}

\begin{table}[t]
\centering
\caption{Quantitative comparison on the CleanBench Subset. Best results are highlighted in \textbf{bold}.}
\label{tab:restoration_main}
\resizebox{0.32\textwidth}{!}{
\begin{tabular}{lccc}
\toprule
\multirow{2}{*}{Method}  & \multicolumn{3}{c}{Avg} \\
\cmidrule(lr){2-4}
 & NIQE $\downarrow$ & BRISQUE $\downarrow$ & CLIP-IQA $\uparrow$ \\
\midrule
DA-CLIP & 5.193 & 15.101 & 0.299 \\
TAO  & \underline{4.966} & \underline{9.716} & 0.310 \\
Gendeg & 5.464 & 18.023 & 0.300 \\
JarvisIR  & 8.121 & 27.440 & \underline{0.443} \\
DFPIR  & 5.023 & 13.997 & 0.296 \\
\midrule
\textbf{Ours} & \textbf{4.683} & \textbf{8.888} & \textbf{0.458} \\
\bottomrule
\end{tabular}}
\vspace{-0.5cm} 
\end{table}

\begin{figure*}[htbp]
    \centering
    \begin{subfigure}[b]{0.24\textwidth}
        \centering
        \includegraphics[width=\textwidth]{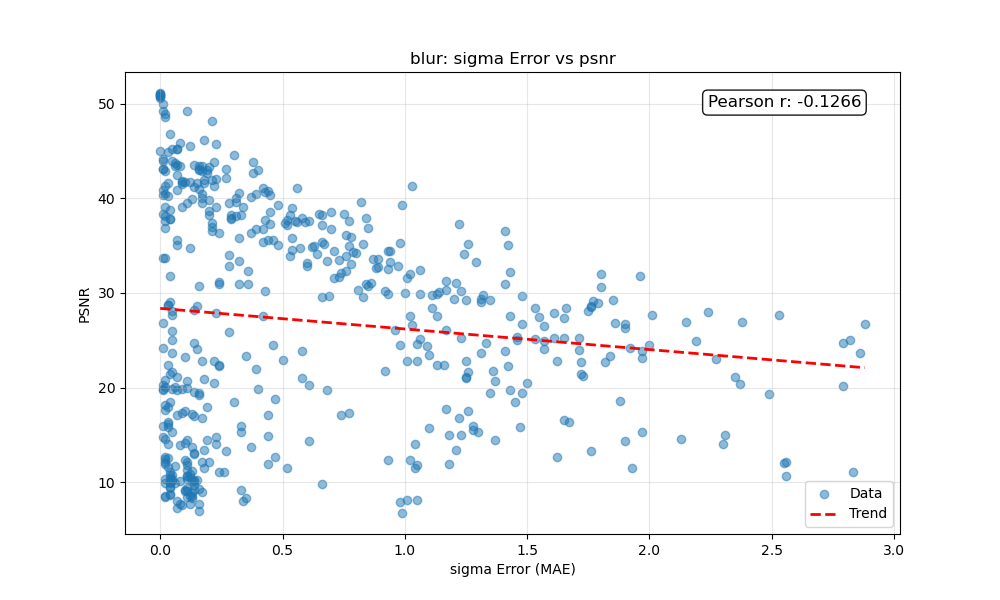}
        \caption{Blur: $\sigma$}
        \label{fig:blur}
    \end{subfigure}
    \hfill
    % 2. Low Res
    \begin{subfigure}[b]{0.24\textwidth}
        \centering
        \includegraphics[width=\textwidth]{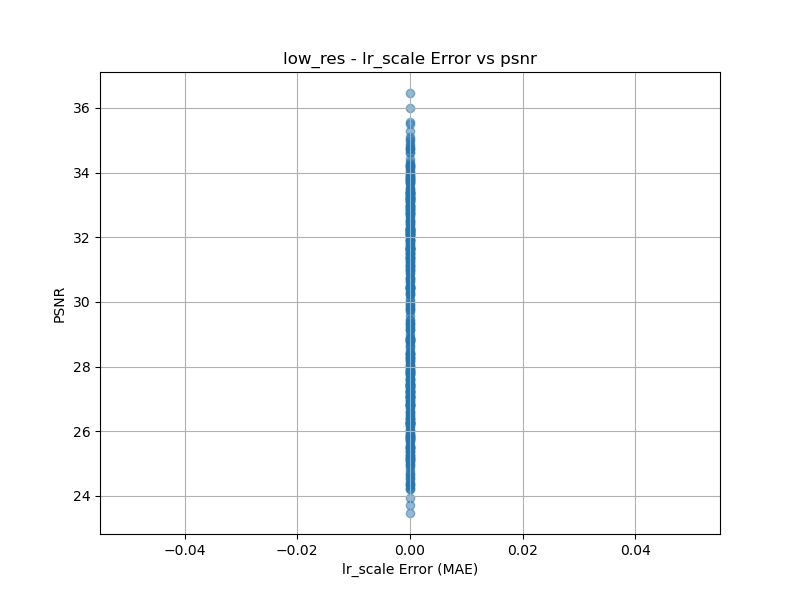}
        \caption{Low Resolution: Scale}
        \label{fig:low_res}
    \end{subfigure}
    \hfill
    % 3. Haze (A)
    \begin{subfigure}[b]{0.24\textwidth}
        \centering
        \includegraphics[width=\textwidth]{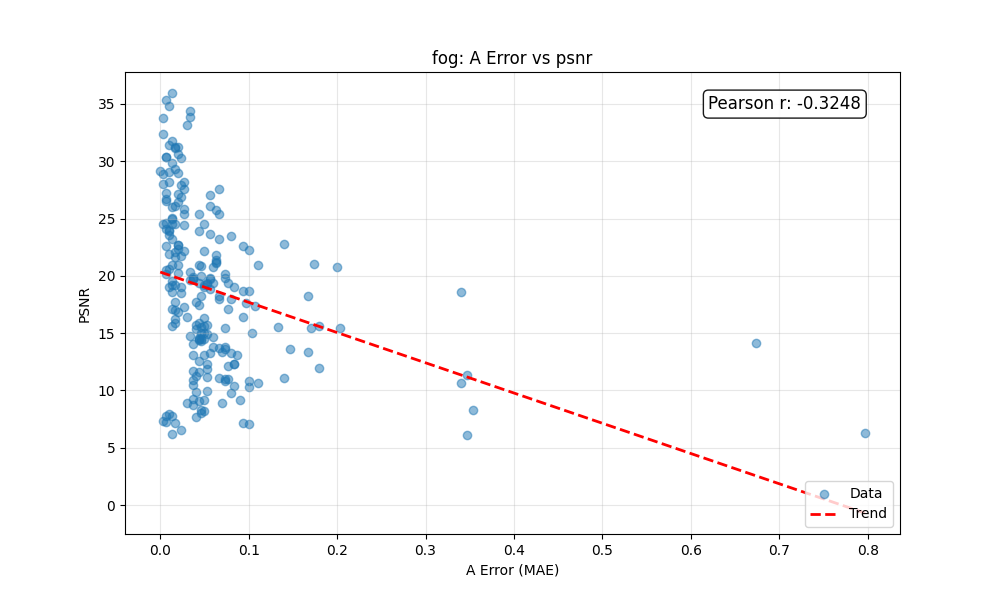}
        \caption{Haze: Atmospheric Light}
        \label{fig:haze_a}
    \end{subfigure}
    \hfill
    % 4. Haze (t)
    \begin{subfigure}[b]{0.24\textwidth}
        \centering
        \includegraphics[width=\textwidth]{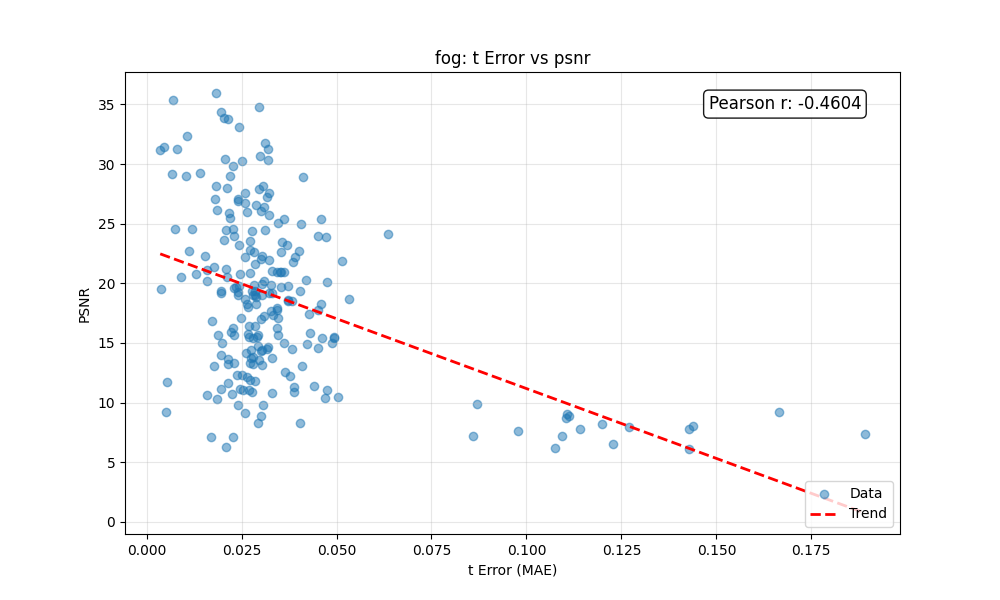}
        \caption{Haze: Transmission Map}
        \label{fig:haze_t}
    \end{subfigure}
    
    \vspace{1em} % 两行之间的垂直间距
    
    % --- 第二行：3张图，大小不变，居中排列 ---
    % 关键点：这里不用 \hfill，而是用固定间距 \hspace，保证它们靠在一起且整体居中
    
    % 1. Night Gain
    \begin{subfigure}[b]{0.24\textwidth}
        \centering
        \includegraphics[width=\textwidth]{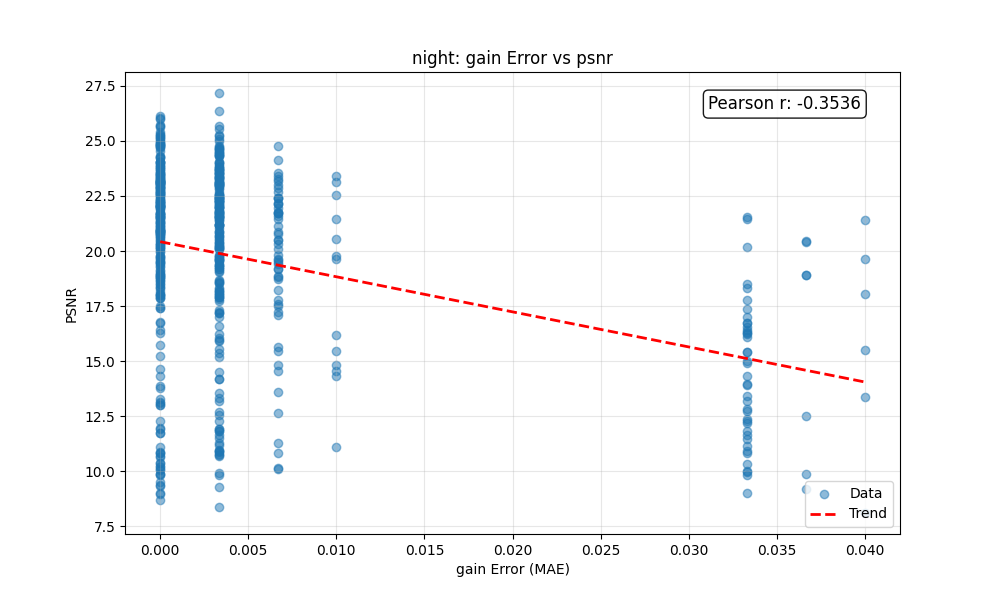}
        \caption{Night: Gain}
        \label{fig:night_gain}
    \end{subfigure}
    \hspace{0.5em} % 图片间距，可视情况微调，推荐 0.5em 到 1em
    % 2. Night Gamma
    \begin{subfigure}[b]{0.24\textwidth}
        \centering
        \includegraphics[width=\textwidth]{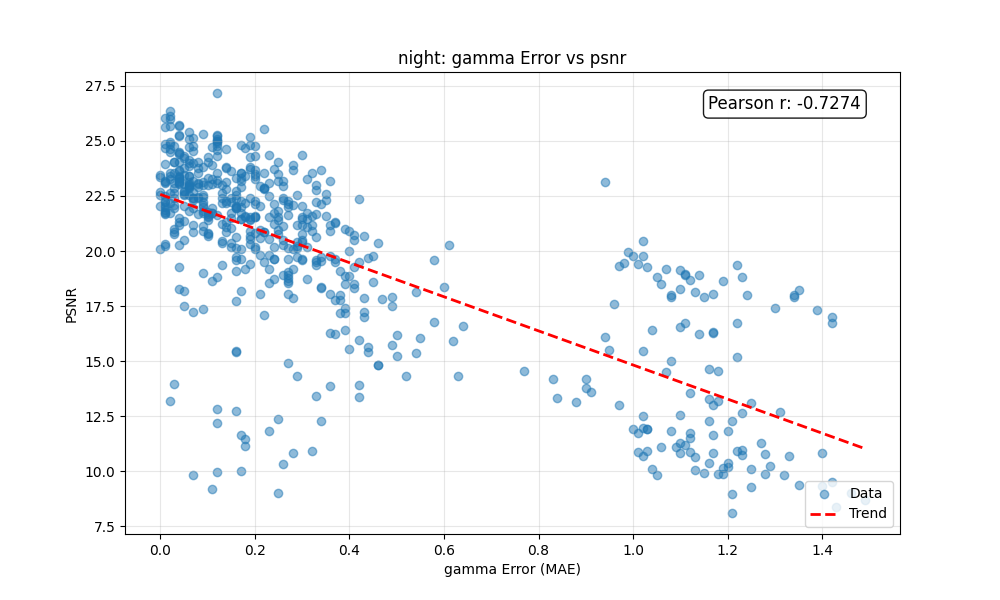}
        \caption{Night: $\gamma$ }
        \label{fig:night_gamma}
    \end{subfigure}
    \hspace{0.5em}
    % 3. Night XYZ
    \begin{subfigure}[b]{0.24\textwidth}
        \centering
        \includegraphics[width=\textwidth]{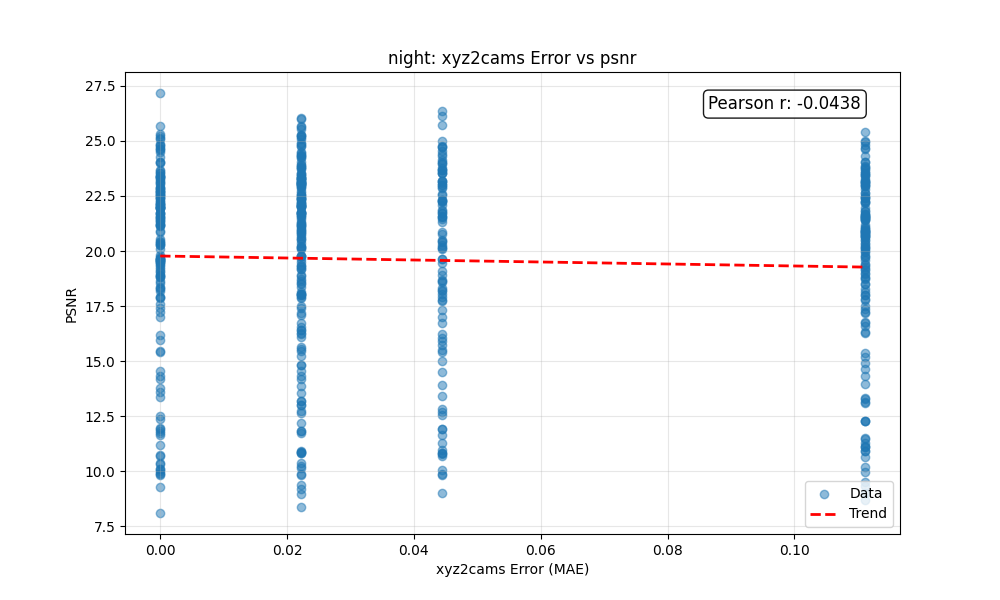}
        \caption{Night: Camera Intrinsic }
        \label{fig:night_xyz}
    \end{subfigure}
    
    \caption{PSNR analysis against parameter estimation errors. Top row: Blur, Low Resolution, and Haze parameters. Bottom row: Night scene parameters. All plots share the same scale for comparison.}
    \label{fig:parameter_errors}
    \vspace{-0.5cm}
\end{figure*}

\textbf{Parameter Value Regression.}
Table \ref{tab:main_results} evaluates the precision of physical parameter estimation, a prerequisite for controllable restoration. 
Off-the-shelf MLLMs (e.g., GPT-4o, Gemini-2.5) exhibit \textit{numerical hallucination}, indicating a fundamental inability to ground visual features into continuous physical quantities without domain-specific alignment. 
While Supervised Fine-Tuning (SFT) significantly bridges this gap—reducing Qwen3VL-8B's average error, but still struggles with fine-grained disentanglement.
In contrast, DU-VLM achieves lowest average P-Rel, validating the efficacy of our structured reasoning framework in capturing mechanistic fidelity.

\subsection{Image Restoration Performance}
\cref{tab:restoration_subtables} presents quantitative results across four distinct degradation types, while \cref{fig:restoration_qualitative} provides qualitative visual comparisons. Our method consistently outperforms existing universal restoration approaches, achieving the best quantitative performance and superior visual fidelity. Specifically, our restorations exhibit strong consistency with ground truth in both texture preservation and color accuracy, effectively eliminating degradations without introducing artifacts or over-smoothing.
Notably, on Night, Haze and  Low Resolution (Tables a, b, d), we achieve superior performance across all metrics.  
On Blur degradation (Table c), our method achieves a substantial PSNR improvement exceeding $8$ dB, which demonstrates the effectiveness of explicit kernel parameter estimation. The SSIM and LPIPS are marginally lower than the best competitors. This reflects the inherent sensitivity of spatially-varying blur to parametric inaccuracies. These results collectively demonstrate that physical parameter enables more precise restoration, particularly when the degradation follows well-defined optical models.

We also evaluate the generalization capability of our method on real-world data from CleanBench~\cite{lin2025jarvisir}, specifically selecting the night and haze subsets for validation. As quantitatively reported in \cref{tab:restoration_main}, our method exhibits strong zero-shot generalization to real-world scenarios, achieving competitive restoration quality without fine-tuning on target distributions.

\subsection{Tolerance to Parametric Errors}
\label{subsec:param_sensitivity}

To quantify how deviations in physical parameter estimation affect restoration fidelity, we analyze the sensitivity of the diffusion model to parametric inaccuracies across the four degradation types. Specifically, we compute the Pearson correlation coefficient $r$ between the absolute parameter estimation error and the resulting restoration quality (PSNR, SSIM, LPIPS) of the diffusion output conditioned on it. \cref{fig:parameter_errors} reports the correlation analysis for PSNR (SSIM and LPIPS results are provided in the Appendix). The results reveal distinct tolerance profiles across degradations: while all parametric errors negatively impact restoration quality, the diffusion model exhibits varying degrees of robustness.

% We attribute this transferability to two synergistic mechanisms. First, the framework promotes \textbf{mechanistic invariance}. By enforcing hierarchical physical constraints (Type $\rightarrow$ Key $\rightarrow$ Value), the model captures universal physical laws—such as atmospheric scattering models—rather than overfitting to synthetic texture biases. This allows for the robust decoupling of degradation parameters from scene semantics even in unseen domains. Second, the \textbf{test-time adaptation} via Online Self-Supervised RL (Sec.~\ref{subsec:rl_optimization}) actively bridges the synthetic-to-real gap. By leveraging no-reference quality feedback, the model dynamically calibrates its parameter estimation policy to the target distribution, significantly elevating zero-shot performance without the need for paired ground truth.
%小羊补充
\subsection{Ablation Studies}
\label{subsec:ablation}

\textbf{Ablation on Multimodal CoT. }
As shown in Table~\ref{tab:ablation_inputs}, removing Multimodal CoT components leads to consistent performance degradation. Specifically, the \textbf{w/o FFT} setting incurs a 2.48 dB drop in PSNR, confirming that frequency-domain information is crucial for identifying subtle artifacts (e.g., compression noise) imperceptible in the spatial domain. Similarly, the \textbf{w/o Edge} configuration degrades Joint Accuracy by 1.9\%, indicating that explicit structural cues are essential for distinguishing high-frequency degradations such as defocus blur. Finally, the \textbf{w/o CoT} variant exhibits significantly higher relative error, demonstrating that anchoring the prediction range via chain-of-thought reasoning substantially improves VLM-based regression accuracy.

% \textbf{Effectiveness of Structured Reinforcement Learning.}
% We further validate the proposed optimization strategy in Table \ref{tab:ablation_rl}. The \emph{SFT Only} baseline suffers from exposure bias, leading to suboptimal parameter precision. While applying \emph{Naive RL} (using a simple scalar reward based on restoration quality) improves the PSNR to 22.89 dB, it fails to significantly boost the physical parameter accuracy (J-Acc 95.2\%). This suggests the model may learn ``shortcuts'' to optimize the restoration metric without truly understanding the degradation physics. In contrast, our \emph{Structured RL}, which factorizes the reward into Type, Key, and Value consistencies, forces the model to adhere to the physical hierarchy. This results in the best performance across all metrics, demonstrating that mechanistic accuracy is a prerequisite for robust restoration.

% \begin{figure}[t]
%     \centering
%     \includegraphics[width=1.0\linewidth]{fig/reward.png}
%     \caption{\textbf{RL reward.} }
%     \label{fig:reward}
% \end{figure}
\textbf{Effectiveness of Structured Reinforcement Learning.}
We validate the proposed optimization strategy in Table \ref{tab:ablation_rl}. The \emph{SFT Only} baseline suffers from exposure bias, leading to suboptimal parameter precision and restoration quality. 
While applying \emph{Naive RL} (optimizing solely for the restoration reward $r_{\text{rec}}$) improves, it fails to boost physical parameter accuracy. This suggests the model learns shortcuts to optimize the perceptual metric without truly understanding the degradation physics. 
In contrast, our \emph{Structured RL}, which factorizes the reward into Type, Key, and Value consistencies, forces the model to adhere to the physical hierarchy. This results in the best performance across all metrics, demonstrating that mechanistic accuracy is a prerequisite for robust restoration.
% Figure \ref{fig:reward} further visualizes the training dynamics. The reward curve shows a steady ascent and stabilizes after approximately 100 steps, confirming that our optimization strategy effectively aligns the model's reasoning process with the compound reward objective.

\begin{table}[ht]
\centering
\caption{Ablation study on \textbf{Multimodel CoT}. We evaluate the impact of different input combinations.}
\label{tab:ablation_inputs}
\resizebox{0.9\linewidth}{!}{
\begin{tabular}{lccc}
\toprule
Configuration & J-Acc ($\uparrow$) & P-Rel ($\downarrow$) & PSNR ($\uparrow$) \\
\midrule
\textbf{Full Model (FFT + Edge + CoT)} & \textbf{99.1\%} & \textbf{0.105} & \textbf{24.32} \\
\midrule
w/o FFT Input (Edge + CoT)             & \underline{98.5\%}          & \underline{0.110}          & 21.84 \\
w/o Edge Input (FFT + CoT)             & 97.2\%          & 0.122          & 22.31 \\
w/o CoT (Direct Prediction)            & 91.4\%          & 0.207          & \underline{22.95} \\
\bottomrule
\end{tabular}
}
\vspace{-0.2cm}
\end{table}

\begin{table}[t]
\centering
\caption{\textbf{Ablation study on Reinforcement Learning strategies.} Comparison between Supervised Fine-Tuning (SFT) and our proposed SFT + Structured RL. We also report Naive RL (optimizing only restoration score) to highlight the importance of physical constraints.}
\label{tab:ablation_rl}
\resizebox{0.95\linewidth}{!}{
\begin{tabular}{lccc}
\toprule
Method & J-Acc ($\uparrow$) & P-Rel ($\downarrow$) & PSNR (dB) ($\uparrow$) \\
\midrule
SFT Only & \underline{96.5\%} & 0.139 & 22.15 \\
SFT + Naive RL & 95.2\% & \underline{0.125} & \underline{22.89} \\
\textbf{SFT + Structured RL (Ours)} & \textbf{99.1\%} & \textbf{0.105} & \textbf{24.32} \\
\bottomrule
\end{tabular}
}
\vspace{-0.4cm}
\end{table}

% \begin{figure}[t]
%     \centering
%     \includegraphics[width=1.0\linewidth]{fig/reward.png}
%     \caption{\textbf{Training convergence of the RL stage.} The solid line represents the mean reward, and the shaded area indicates the standard deviation. The reward stabilizes rapidly, confirming the effectiveness and stability of the structured optimization.}
%     \label{fig:reward}
%     \vspace{-0.3cm}
% \end{figure}

\begin{figure}[t]
    \centering
    \includegraphics[width=1.0\linewidth]{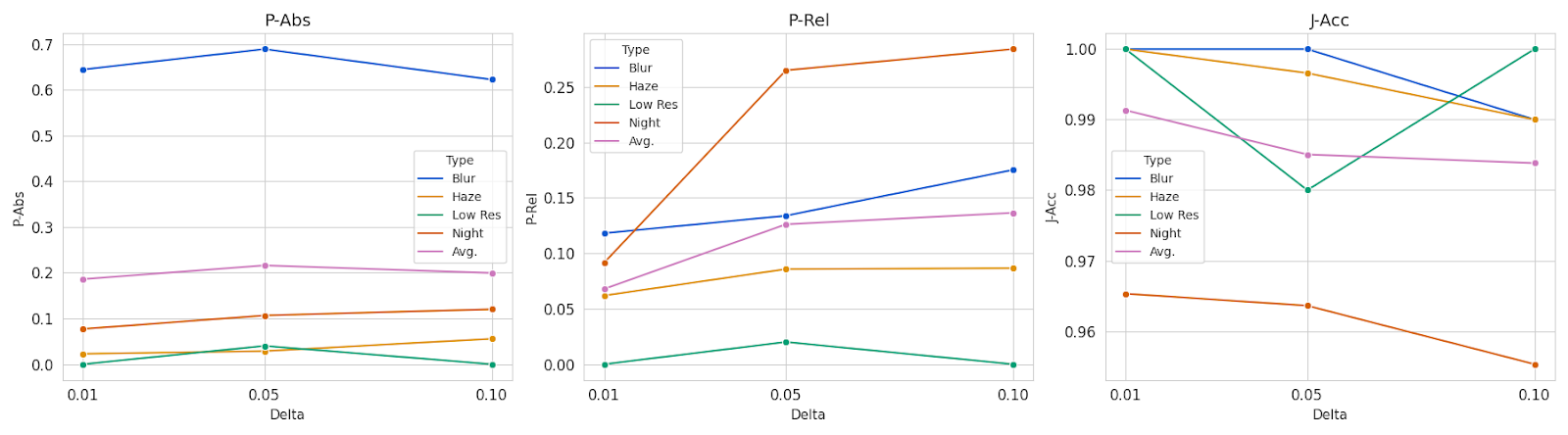}
    \caption{\textbf{Ablation study on quantization grid $\Delta$}. We evaluate the model's performance across different bin ($\Delta \in \{0.01, 0.05, 0.1\}$) in terms of parameter estimation errors: P-Abs, P-Rel and J-Acc.}
    \label{fig:delta_ablation}
    \vspace{-0.5cm}
\end{figure}
%%%%%%%%%%%%%%%%%%%%%%%%%%%%%%%%%%%%%%%%%%
\textbf{Ablation on quantization grid $\Delta$.}
Figure~\ref{fig:delta_ablation} examines the effect of quantization granularity. Coarse grids ($\Delta \in \{0.1, 0.05\}$) restrict resolution, resulting in elevated errors (P-Rel $> 0.12$). Finer quantization ($\Delta=0.01$) significantly improves precision, minimizing both P-Rel and P-Abs metrics. These results corroborate the theoretical guarantees established in Proposition \ref{pro:equivalence} and Proposition \ref{pro:bounds}.

\subsection{Limitations and Future Work}
\label{sec:limitations_future}
\textbf{\textit{Which Direction Should We Move Towards?}}

\textit{Scalability of Degradation Modeling.} 
While DU-110k provides high-quality annotations, real-world degradations require more diverse, hybrid physical formulations. Future work should extend to composite degradations and enable automated model selection, allowing the VLM to dynamically compose physical models for unseen patterns.

\textit{Bridging the Sim-to-Real Gap.}
Two aspects warrant investigation: (1) Developing more accurate physical models that better approximate real-world imaging physics beyond current simplifications; (2) Recording actual imaging pipeline parameters (\eg, camera ISP settings, atmospheric conditions) during real-world acquisition to construct authentic clean-degraded pairs with grounded physical annotations.

\textit{Towards End-to-End Joint Optimization.}
Our decoupled approach trains VLM for frozen diffusion backbones. Future work should investigate joint optimization via differentiable diffusion sampling or surrogate gradients, enabling simultaneous fine-tuning of both modules.

\section{Conclusion}
\label{sec:conclusion}
We introduced DU-VLM, a framework establishing a paradigm shift from qualitative degradation description to \textit{\textbf{parametric physical understanding}}. By formulating degradation modeling as a hierarchical structured prediction problem, we theoretically unified classification and value regression within a single autoregressive objective. Supported by our proposed DU-110k benchmark, we demonstrated that explicit physical reasoning enables VLMs to perform effective zero-shot steering of diffusion-based restoration. This work bridges the gap between semantic reasoning and low-level physics, offering a principled path toward interpretable and controllable vision systems.

% Our approach goes beyond improved metrics; by integrating Multimodal Chain-of-Thought reasoning with Structured Reinforcement Learning, we effectively bridged the gap between semantic visual representations and continuous physical quantities. Crucially, we showed that these interpretable parameters serve as powerful, zero-shot prompts for guiding pre-trained diffusion models, achieving state-of-the-art restoration without task-specific tuning. We believe this work establishes a principled foundation for \textbf{physics-aware vision systems}, paving the way for interpretable AI that not only perceives visual artifacts but fundamentally understands their underlying causes.

% \begin{figure}[h]
% \centering
% \includegraphics[width=\linewidth]{figures/generalization.png}
% \caption{Generalization performance on unseen degradation types. Our approach maintains strong performance even for degradation types not seen during training.}
% \label{fig:generalization}
% \end{figure}

% \subsection{Ablation Studies}
% \begin{table}[h]
% \centering
% \caption{Ablation study on key components. CoT: Chain-of-Thought, SR: Structured Reward.}
% \label{tab:ablation}
% \begin{tabular}{lccc}
% \toprule
% Configuration & PSNR & SSIM & Param. Acc. \\
% \midrule
% w/o CoT & 21.34 & 0.871 & 82.1\% \\
% w/o Structured Reward & 22.01 & 0.882 & 84.5\% \\
% w/o Physics Constraint & 22.56 & 0.889 & 85.2\% \\
% \midrule
% \textbf{Full Model} & \textbf{23.67} & \textbf{0.903} & \textbf{87.3\%} \\
% \bottomrule
% \end{tabular}
% \end{table}

\section*{Broader Impact Statement}

This work advances parametric image degradation understanding and restoration through Vision-Language Models. While our research aims to improve image quality assessment and restoration fidelity, we acknowledge several societal implications that warrant discussion.

\textbf{Positive Implications.} Our approach enables principled, physics-aware image restoration that could benefit healthcare (medical imaging enhancement), cultural heritage (historical photo restoration), and and public safety (degraded surveillance footage recovery). By grounding restoration in explicit physical parameters rather than black-box transformations, DU-VLM offers greater transparency and controllability, potentially reducing arbitrary distortions in critical applications. The release of DU-110k provides a standardized benchmark for the research community to rigorously evaluate degradation understanding, potentially advancing low-light photography, atmospheric dehazing, and lens aberration correction for consumer and professional imaging.

\textbf{Ethical Considerations and Risks.} As with any powerful image restoration technology, DU-VLM carries risks of misuse. High-fidelity restoration capabilities could be exploited to:
\begin{itemize}
    \item Conceal manipulation traces in forensic contexts (e.g., masking splicing artifacts by attributing them to benign degradations);
    \item Circumvent privacy protections by enhancing low-quality surveillance or unintended background captures in public spaces;
    \item Generate synthetic training data for deepfakes with realistic degradation profiles to evade detection systems.
\end{itemize}

\textbf{Mitigation and Responsible Use.} We emphasize that DU-VLM is intended for legitimate image enhancement and analysis tasks where the original content ownership is clear. We encourage the research community to develop forensic techniques alongside restoration methods to detect parametric tampering. Furthermore, while our work improves restoration quality, it does not create novel imagery from scratch; rather, it recovers underlying information. We advocate for clear guidelines regarding consent when enhancing images of individuals and transparency regarding AI-based restoration in journalistic or legal contexts.

%% file: sec/X_appendix.tex
\newpage
\appendix
\onecolumn
\section{Proofs of Theoretical Analysis}
\label{app:proofs}

In this section, we provide detailed derivations for Proposition \ref{pro:equivalence} and Proposition \ref{pro:bounds} presented in the main text.

\subsection{Setup and Notation}

Recall the structured output definition $\hat{\mathcal{D}} = \{(\hat{t}, \hat{k}, \hat{v})\}$. We define the joint distribution factorization as:
\begin{equation}
    p_\theta(t, k, v \mid \bm{x}) = p_\theta(t \mid \bm{x}) \cdot p_\theta(k \mid t, \bm{x}) \cdot p_\theta(v \mid t, k, \bm{x}).
\end{equation}
Let the continuous value $v \in \mathbb{R}^d$ be bounded within a domain $\mathcal{V}$. We employ a quantization scheme $\mathcal{Q}: \mathcal{V} \to \{1, \dots, Z\}$ that maps a continuous vector $v$ to a discrete bin index $z$. Let $\mathcal{B}_z$ denote the region of bin $z$ with uniform volume (Lebesgue measure) $\text{Vol}(\mathcal{B}_z) = \Delta$. The discrete target for the NTP model corresponds to the tuple $(t, k, z)$.

\subsection{Proof of Proposition \ref{pro:equivalence}}

\textbf{Proposition \ref{pro:equivalence}} \textit{Under the assumptions of local Gaussianity and fine quantization, $\mathcal{L}_{\text{NTP}}$ decomposes into structure classification and value regression losses.}

\begin{proof}
The Negative Log-Likelihood (NLL) for the Next-Token Prediction objective is given by:
\begin{equation}
    \mathcal{L}_{\text{NTP}}(\theta) = -\log P_\theta(t, k, z \mid \bm{x}).
\end{equation}
Using the chain rule of probability, we expand this term:
\begin{align}
    \mathcal{L}_{\text{NTP}}(\theta) &= -\log \left[ P_\theta(t \mid \bm{x}) P_\theta(k \mid t, \bm{x}) P_\theta(z \mid t, k, \bm{x}) \right] \\
    &= \underbrace{-\log P_\theta(t, k \mid \bm{x})}_{\mathcal{L}_{\text{struct}}} \underbrace{-\log P_\theta(z \mid t, k, \bm{x})}_{\mathcal{L}_{\text{val}}}.
\end{align}
The first term $\mathcal{L}_{\text{struct}}$ corresponds exactly to the cross-entropy loss for classification of the structure components ($t, k$). We now analyze the second term $\mathcal{L}_{\text{val}}$.

By the Mean Value Theorem for integrals, the probability mass of the discrete bin $z$ is related to the continuous probability density $p_\theta(v \mid t, k, \bm{x})$ as follows:
\begin{equation}
    P_\theta(z \mid t, k, \bm{x}) = \int_{u \in \mathcal{B}_z} p_\theta(u \mid t, k, \bm{x}) \, du \approx p_\theta(v \mid t, k, \bm{x}) \cdot \Delta,
\end{equation}
where $v$ is a representative point (e.g., the center) of bin $\mathcal{B}_z$. This approximation becomes exact as $\Delta \to 0$.

Substituting this into $\mathcal{L}_{\text{val}}$:
\begin{equation}
    \mathcal{L}_{\text{val}} \approx -\log \left( p_\theta(v \mid t, k, \bm{x}) \cdot \Delta \right) = -\log p_\theta(v \mid t, k, \bm{x}) - \log \Delta.
\end{equation}
We invoke the assumption that the local conditional density follows a Gaussian distribution $\mathcal{N}(\mu_\theta(\cdot), \sigma^2 \mathbf{I})$:
\begin{equation}
    p_\theta(v \mid t, k, \bm{x}) = \frac{1}{(2\pi\sigma^2)^{d/2}} \exp \left( -\frac{\|v - \mu_\theta(t, k, \bm{x})\|^2}{2\sigma^2} \right).
\end{equation}
Taking the negative log:
\begin{equation}
    -\log p_\theta(v \mid t, k, \bm{x}) = \frac{1}{2\sigma^2} \|v - \mu_\theta(t, k, \bm{x})\|^2 + \frac{d}{2} \log(2\pi\sigma^2).
\end{equation}
Combining these results, the total loss becomes:
\begin{equation}
    \mathcal{L}_{\text{NTP}} \approx \mathcal{L}_{\text{struct}} + \frac{1}{2\sigma^2} \|v - \mu_\theta\|^2 + \underbrace{\frac{d}{2} \log(2\pi\sigma^2) - \log \Delta}_{C}.
\end{equation}
Since $\Delta$ and $\sigma$ are fixed hyperparameters during optimization, $C$ is a constant. Thus, minimizing $\mathcal{L}_{\text{NTP}}$ is equivalent to minimizing the sum of classification and weighted regression losses.
\end{proof}

\subsection{Proof of Proposition \ref{pro:bounds}}

\textbf{Proposition \ref{pro:bounds}: } \textit{Given an excess risk $\epsilon$, the classification error $R_{\text{cls}}$ and regression MSE $R_{\text{reg}}$ are bounded by $\sqrt{2\epsilon}$ and $\Delta^2/4 + D_{\max}^2\sqrt{2\epsilon}$, respectively.}

\begin{proof}
Let $P^*$ denote the true distribution and $P_\theta$ denote the model distribution. The excess risk condition is given by:
\begin{equation}
    \mathbb{E}_{\bm{x}} \left[ D_{\text{KL}}(P^*(\cdot \mid \bm{x}) \,\|\, P_\theta(\cdot \mid \bm{x})) \right] \leq \epsilon.
\end{equation}

\textbf{Part 1: Classification Error Bound}

The classification error corresponds to the event where the predicted structure tuple $S = (t, k)$ does not match the ground truth $S^*$. The Total Variation (TV) distance relates to the probability of this event. By Pinsker's Inequality, for any input $\bm{x}$:
\begin{equation}
    \delta_{\text{TV}}(P^*, P_\theta) \leq \sqrt{\frac{1}{2} D_{\text{KL}}(P^* \,\|\, P_\theta)}.
\end{equation}
A well-known corollary relates the $L_1$ distance to KL divergence: $\|P^* - P_\theta\|_1 \leq \sqrt{2 D_{\text{KL}}(P^* \,\|\, P_\theta)}$.
The misclassification rate $R_{\text{cls}}$ is bounded by the $L_1$ distance between the true and predicted distributions. Using Jensen's inequality ($\mathbb{E}[\sqrt{X}] \leq \sqrt{\mathbb{E}[X]}$):
\begin{equation}
    R_{\text{cls}} \leq \mathbb{E}_{\bm{x}} \left[ \sqrt{2 D_{\text{KL}}} \right] \leq \sqrt{2 \mathbb{E}_{\bm{x}} [D_{\text{KL}}]} \leq \sqrt{2\epsilon}.
\end{equation}

\textbf{Part 2: Regression Error Bound}

Let $\hat{v}$ be the centroid of the predicted bin $\hat{z}$. The regression Mean Squared Error (MSE) is:
\begin{equation}
    R_{\text{reg}} = \mathbb{E} [\|v - \hat{v}\|^2].
\end{equation}
We decompose this expectation based on whether the structure classification (and consequently the coarse bin selection) is correct. Let $E$ be the event that the correct bin is selected (i.e., $\hat{z} = z^*$).
\begin{equation}
    \mathbb{E} [\|v - \hat{v}\|^2] = P(E) \cdot \mathbb{E} [\|v - \hat{v}\|^2 \mid E] + P(E^c) \cdot \mathbb{E} [\|v - \hat{v}\|^2 \mid E^c].
\end{equation}

\textit{Case 1: Correct Bin ($E$)}. If the correct bin is selected, both $v$ and $\hat{v}$ lie within $\mathcal{B}_{z^*}$. Since the maximum diagonal length of the bin is $\Delta$, and $\hat{v}$ is the centroid, the maximum distance is $\Delta/2$. Thus:
\begin{equation}
    \mathbb{E} [\|v - \hat{v}\|^2 \mid E] \leq \left(\frac{\Delta}{2}\right)^2 = \frac{\Delta^2}{4}.
\end{equation}

\textit{Case 2: Incorrect Bin ($E^c$)}. If the bin is incorrect, the error is bounded by the diameter of the value space $\mathcal{V}$, denoted $D_{\max}$:
\begin{equation}
    \mathbb{E} [\|v - \hat{v}\|^2 \mid E^c] \leq D_{\max}^2.
\end{equation}

Combining these, and noting that $P(E) \leq 1$ and $P(E^c) = R_{\text{cls}} \leq \sqrt{2\epsilon}$:
\begin{align}
    R_{\text{reg}} &\leq 1 \cdot \frac{\Delta^2}{4} + \sqrt{2\epsilon} \cdot D_{\max}^2 \\
    &= \frac{\Delta^2}{4} + D_{\max}^2 \sqrt{2\epsilon}.
\end{align}
This concludes the proof.
\end{proof}

\section{Detailed Error Analysis of Generalist MLLMs}
\label{sec:appendix_error_analysis}

To further investigate the limitations of generalist Multimodal Large Language Models (MLLMs) in degradation understanding, we present the confusion matrices for five representative models: GPT-4o, Gemini-2.5, Qwen-VL (8B/32B), and LLaVA-1.5. These experiments were conducted on the \textbf{DU-110k} test set under a zero-shot setting. The visualization results, shown in Figure \ref{fig:appendix_confusion}, reveal three distinct failure modes that underscore the necessity of our proposed DU-VLM.

\begin{figure*}[t]
    \centering
    % Row 1: 3 images
    \begin{subfigure}[b]{0.33\textwidth}
        \centering
        \includegraphics[width=\linewidth]{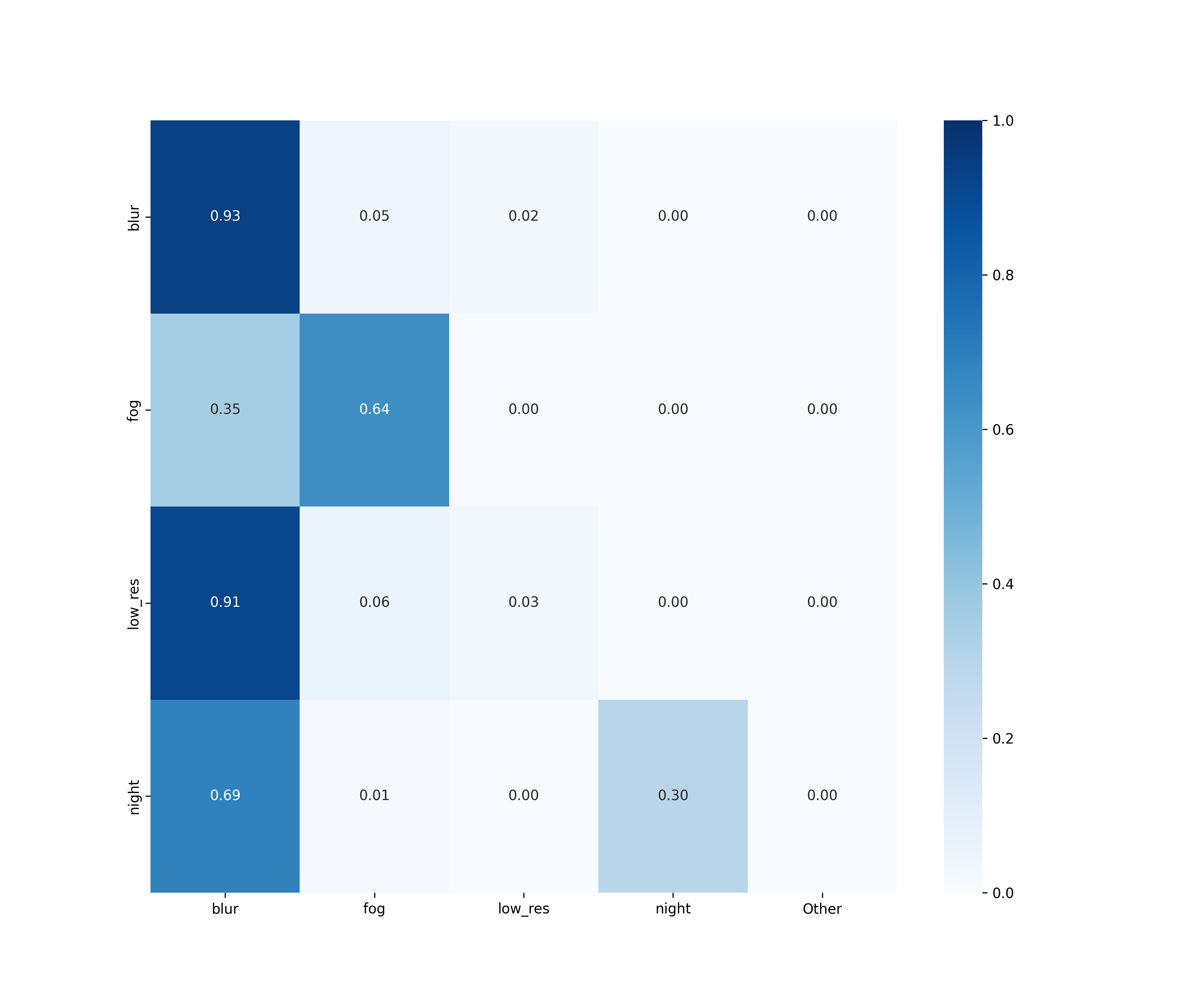}
        \caption{GPT-4o}
        \label{fig:cm_gpt4o}
    \end{subfigure}
    \hfill
    \begin{subfigure}[b]{0.33\textwidth}
        \centering
        \includegraphics[width=\linewidth]{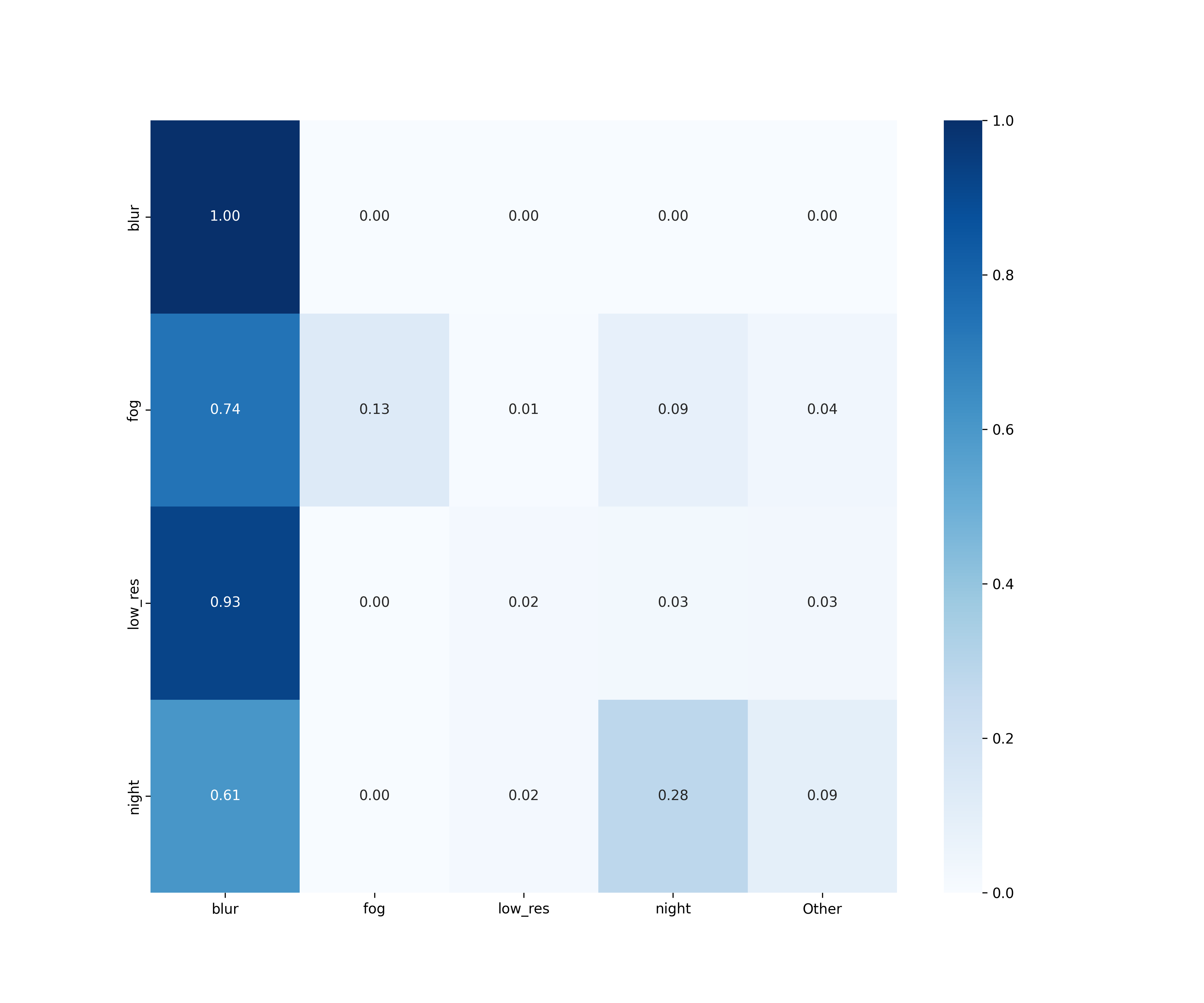}
        \caption{Qwen-VL-8B}
        \label{fig:cm_qwen8b}
    \end{subfigure}
    \hfill
    \begin{subfigure}[b]{0.33\textwidth}
        \centering
        \includegraphics[width=\linewidth]{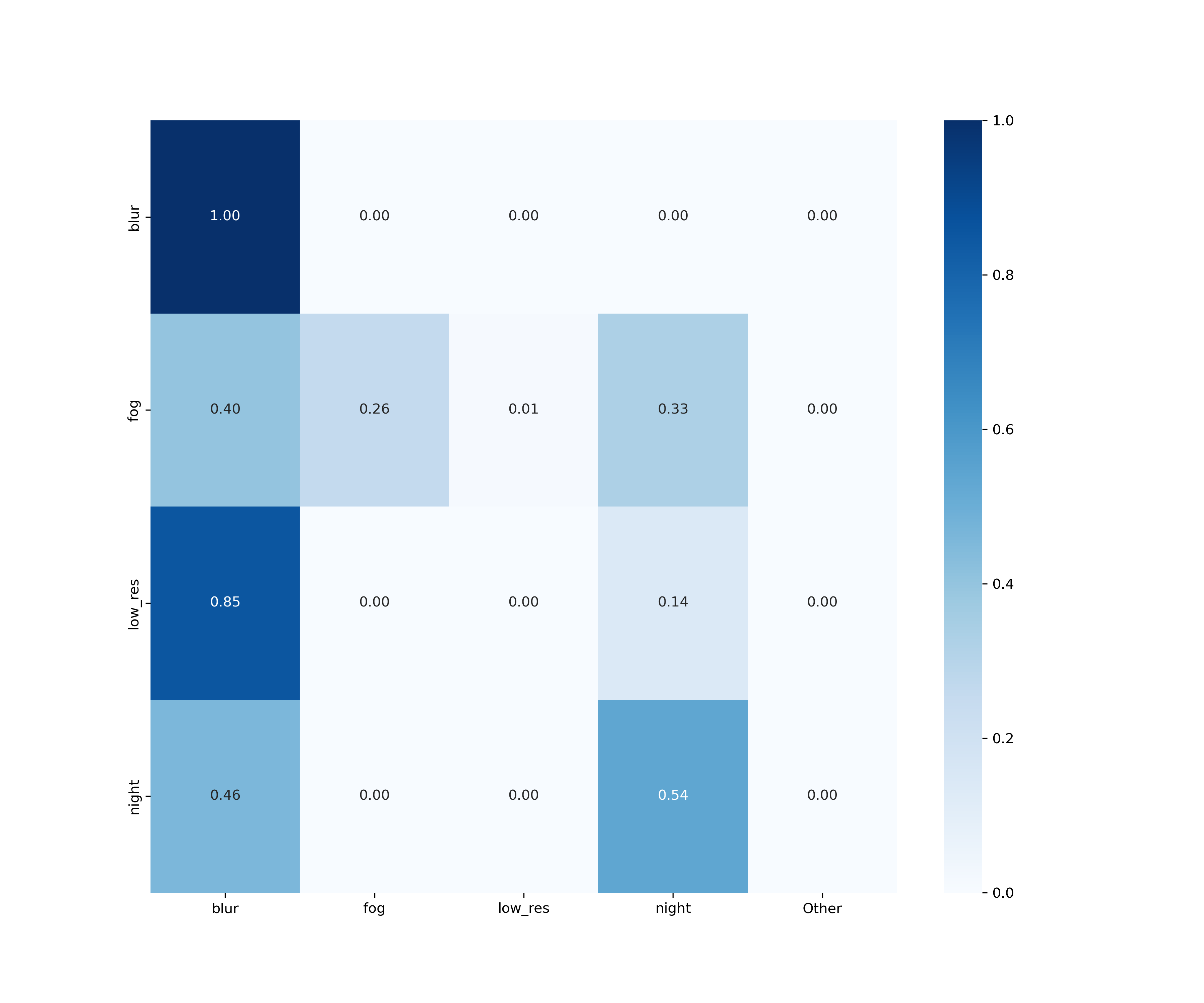}
        \caption{Qwen-VL-32B}
        \label{fig:cm_qwen32b}
    \end{subfigure}
    
    \vspace{0.3cm} % Spacing between rows
    
    % Row 2: 2 images, centered with spacing
    \begin{subfigure}[b]{0.33\textwidth}
        \centering
        \includegraphics[width=\linewidth]{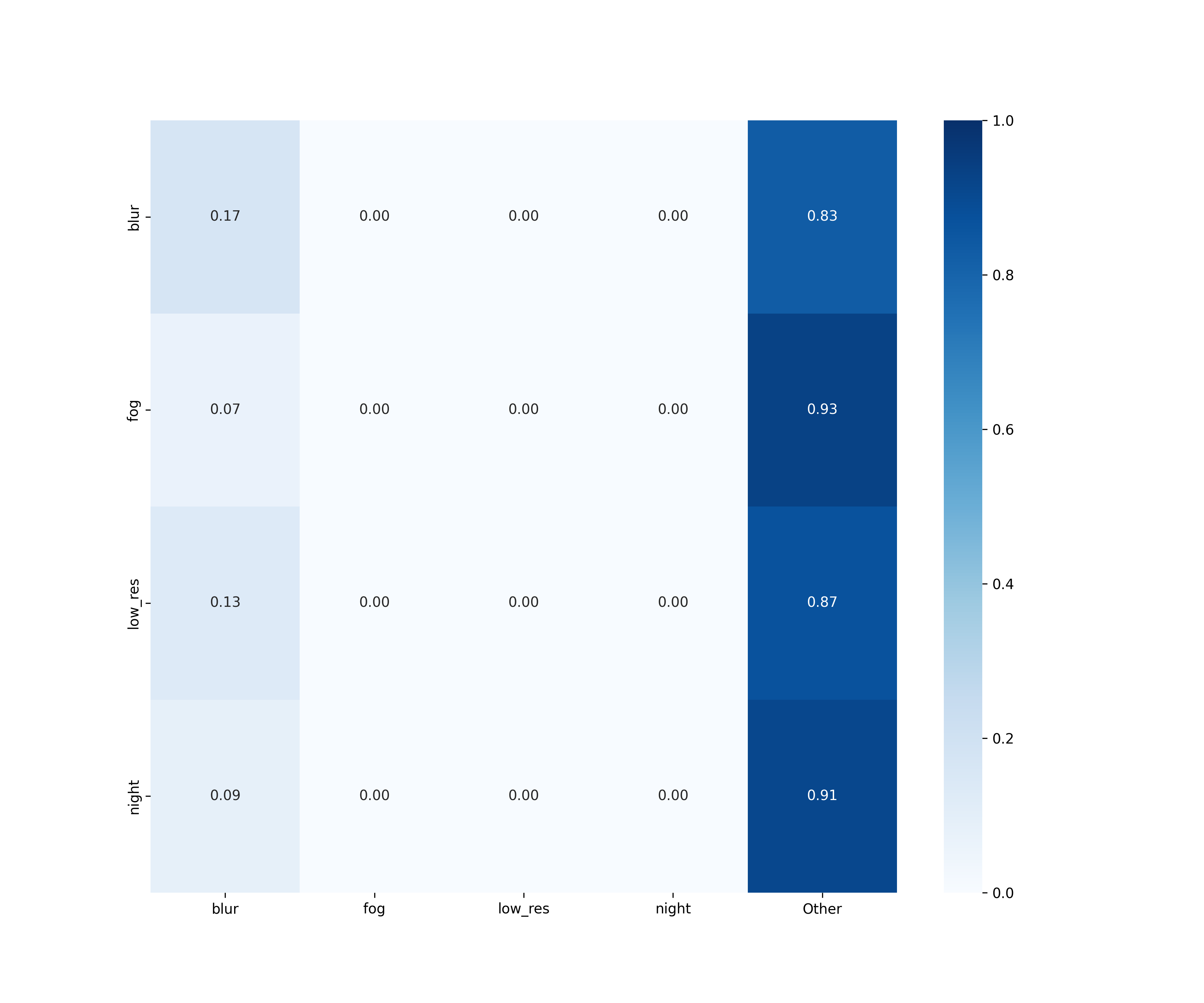}
        \caption{LLaVA-1.5-7B}
        \label{fig:cm_llava}
    \end{subfigure}
    \hfill
    \begin{subfigure}[b]{0.33\textwidth}
        \centering
        \includegraphics[width=\linewidth]{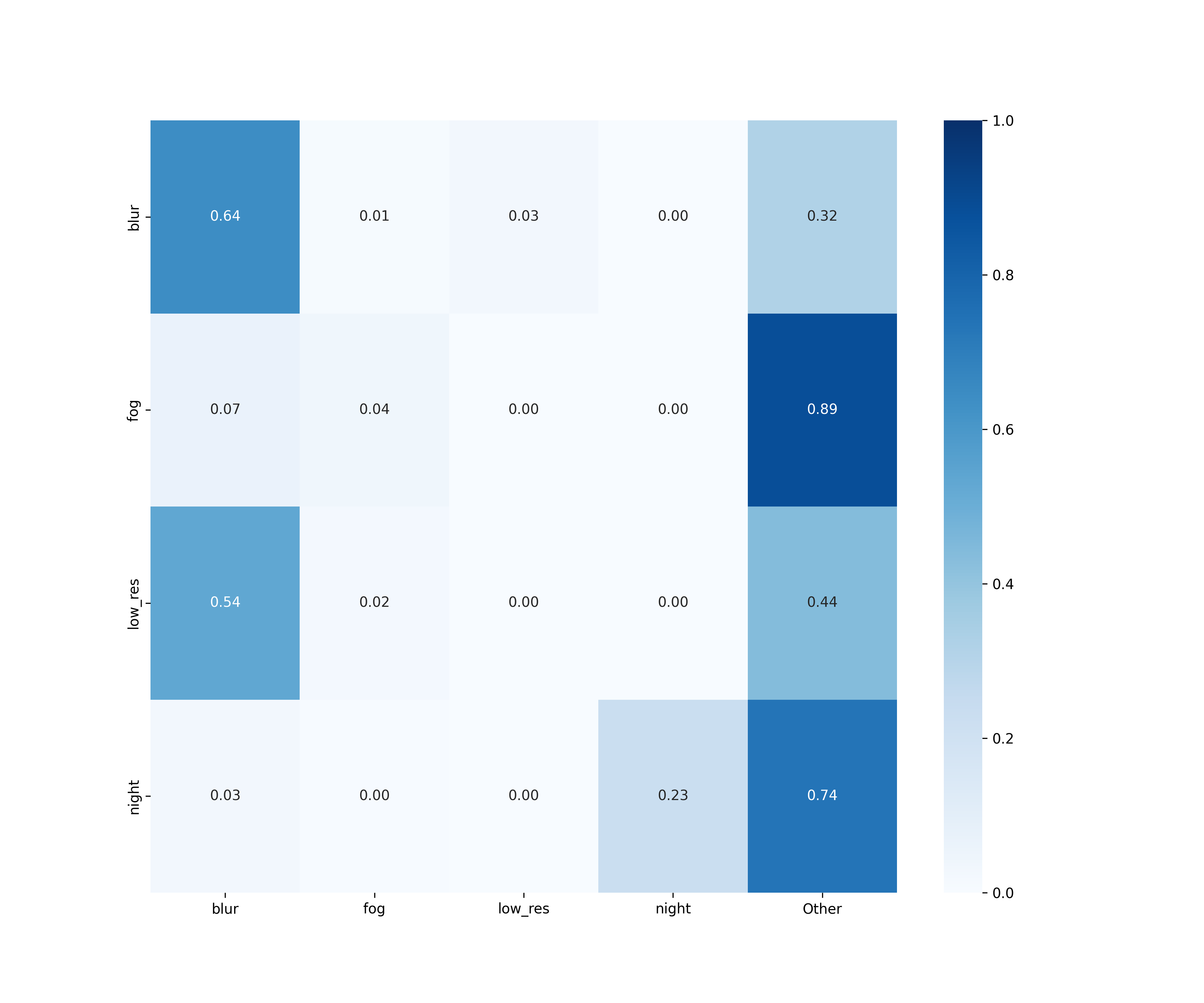}
        \caption{Gemini-2.5}
        \label{fig:cm_gemini}
    \end{subfigure}
    \hfill
    \begin{subfigure}[b]{0.33\textwidth}
        \centering
        \includegraphics[width=\linewidth]{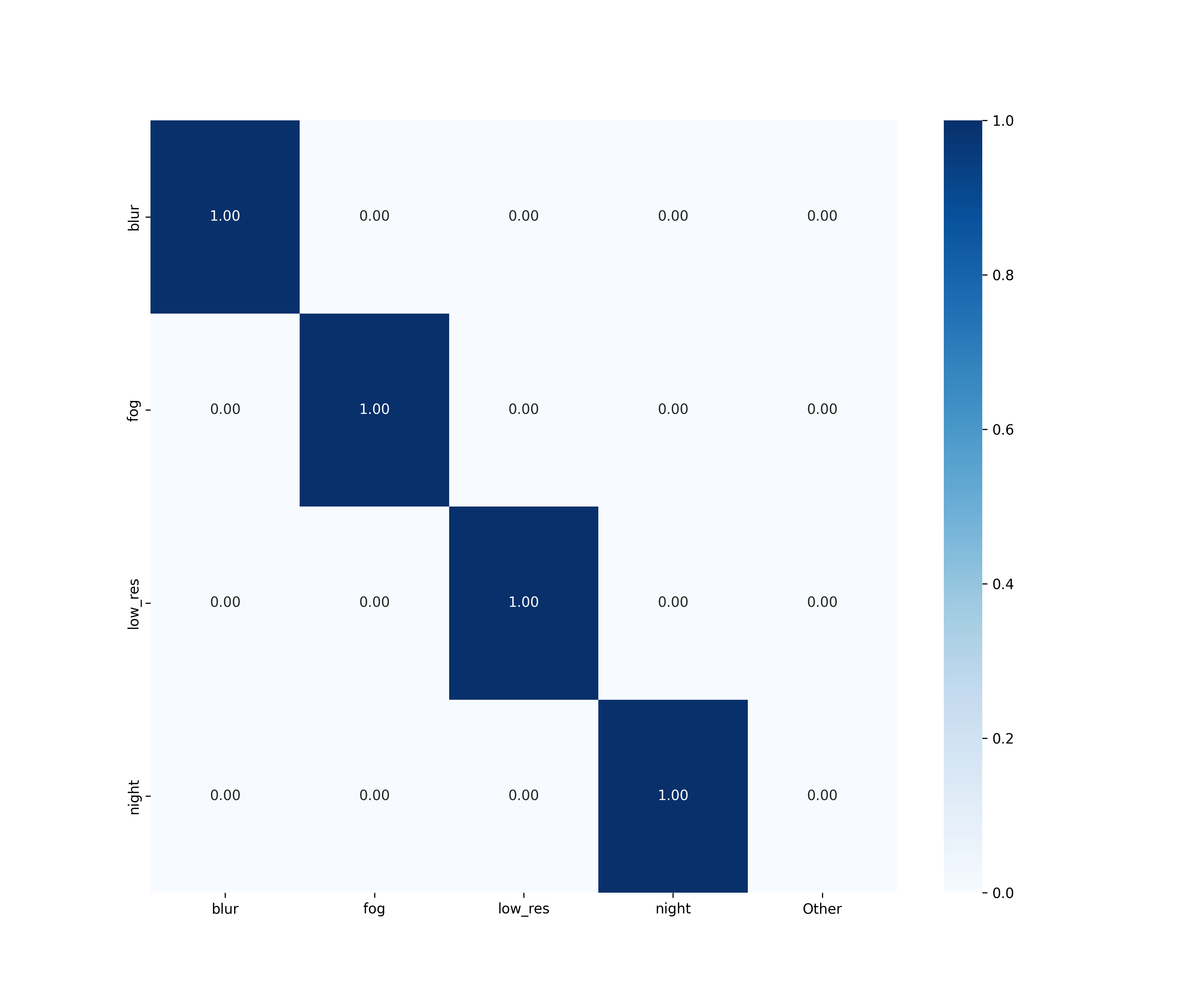}
        \caption{DU-VLM}
        \label{fig:cm_du}
    \end{subfigure}

    \caption{\textbf{Confusion Matrices of Generalist MLLMs under Zero-shot Setting.} The analysis reveals severe semantic biases. (a-c) Models exhibit a strong \textit{``Blur'' Bias}, classifying distinct degradations (e.g., Low-Resolution, Fog) as generic Blur. (d) LLaVA suffers from \textit{Mode Collapse}, predominantly predicting ``Color'' for all inputs. (e) Gemini shows \textit{Taxonomy Misalignment}, frequently categorizing valid physical degradations as ``Other''.}
    \label{fig:appendix_confusion}
\end{figure*}

\paragraph{1. Semantic Over-simplification (The ``Blur'' Bias).}
As illustrated in Figure \ref{fig:cm_gpt4o}, \ref{fig:cm_qwen8b}, and \ref{fig:cm_qwen32b}, highly capable models such as GPT-4o and the Qwen-VL series exhibit a dominant bias towards the \textit{Blur} category.
\begin{itemize}
    \item \textbf{GPT-4o} correctly identifies Blur (556 samples) but misclassifies 73\% of \textit{Color} samples (587/800 estimated) and 77\% of \textit{Low-Resolution} samples (430/555 estimated) as Blur.
    \item \textbf{Qwen-VL-8B} displays an even more severe bias, collapsing nearly the entire \textit{Color}, \textit{Fog}, and \textit{Low-Resolution} categories into the \textit{Blur} column (e.g., 555 Low-Res samples misclassified as Blur). 
    \item While \textbf{Qwen-VL-32B} shows slight improvement in separating \textit{Night} scenes, the confusion between \textit{Low-Resolution} and \textit{Blur} remains unresolved.
\end{itemize}
This phenomenon suggests that generalist models perceive ``degradation'' generically as a loss of high-frequency details (sharpness), lacking the fine-grained physical knowledge to distinguish between optical defocus (Blur), downsampling artifacts (Low-Res), and atmospheric scattering (Fog).

\paragraph{2. Mode Collapse.}
Figure \ref{fig:cm_llava} demonstrates a catastrophic failure case in LLaVA-1.5. The model exhibits extreme mode collapse, predicting the \textit{Color} category for the vast majority of inputs regardless of the actual degradation type (e.g., 413 Fog samples and 412 Night samples are misclassified as Color). This indicates that without specific instruction tuning or physical grounding, standard VLM visual encoders may fail to attend to the relevant low-level distortion features.

\paragraph{3. Taxonomy Misalignment.}
In the case of Gemini-2.5 (Figure \ref{fig:cm_gemini}), we observe a different behavior characterized by uncertainty. A significant portion of \textit{Fog} (468 samples) and \textit{Night} (442 samples) degradations are classified as ``Other''. This suggests that while the model avoids the ``Blur'' bias, it fails to align its internal semantic representations with the specific physical taxonomy used in restoration tasks, likely due to safety filters or a lack of domain-specific definition in its pre-training data.

\vspace{0.5cm}
\noindent\textbf{Conclusion.} These visualizations empirically validate that general-purpose MLLMs are ill-suited for precision-critical restoration guidance. They tend to hallucinate degradation types or collapse into generic categories. In contrast, our \textbf{DU-VLM}, constrained by the proposed hierarchical autoregressive objective, effectively disentangles these physical factors as evidenced by the high Top-1 accuracy reported in the main paper.

\section{Analysis of Parameter-Level Structural Hallucination}
\label{sec:appendix_param_analysis}

Beyond classification errors, mechanistic understanding requires the correct prediction of degradation-specific parameter keys (\textit{e.g.}, predicting \texttt{sigma} for Blur, but \texttt{scale} for Low-Resolution). In Figure \ref{fig:appendix_param}, we visualize the distribution of parameter keys predicted by generalist MLLMs for each ground-truth degradation type. This analysis exposes a critical ``One-Size-Fits-All'' fallacy in generalist models.

\paragraph{1. The Sigma Hallucination (GPT-4o \& Qwen Series).} 
As seen in Fig. \ref{fig:param_gpt}, \ref{fig:param_qwen8b}, and \ref{fig:param_qwen32b}, these models predominantly predict the \texttt{sigma} key (associated with Gaussian Blur) across physically unrelated categories. Specifically, for \textit{Low-Resolution} inputs (pixelation artifacts), GPT-4o predicts \texttt{sigma} (light blue bar) in over 70\% of cases. Physically, applying a deblurring kernel to a downsampled image often exacerbates artifacts (ringing). Qwen-8B (Fig. \ref{fig:param_qwen8b}) exhibits this bias most severely, predicting \texttt{sigma} (yellow bar) for nearly 100\% of \textit{Color} and \textit{Low-Resolution} samples.

\paragraph{2. Photometric Overfitting (LLaVA).} 
Consistent with its classification mode collapse, LLaVA (Fig. \ref{fig:param_llava}) assigns \texttt{gamma}-related keys (pink bar) to structural degradations like \textit{Low-Resolution} and \textit{Blur}. This implies the model attempts to ``fix'' geometric information loss using photometric color correction, which is theoretically impossible.

\paragraph{3. The Scale of Uncertainty (Qwen-235B).} 
Interestingly, the larger Qwen-235B (Fig. \ref{fig:param_qwen235b}) shows a high rate of ``None/Empty'' predictions (orange bar) for non-blur categories. While it avoids the ``Sigma'' hallucination to some extent, it fails to identify the correct physical keys (e.g., missing \texttt{A, t} for Fog), rendering it useless for controllable restoration.

These findings confirm that without the proposed hierarchical autoregressive modeling, MLLMs fail to disentangle the \textit{cause} (physics) from the \textit{effect} (quality drop).

\begin{figure*}[t]
    \centering
    % Row 1: GPT-4o & LLaVA
    \begin{subfigure}[b]{0.48\textwidth}
        \centering
        \includegraphics[width=\linewidth]{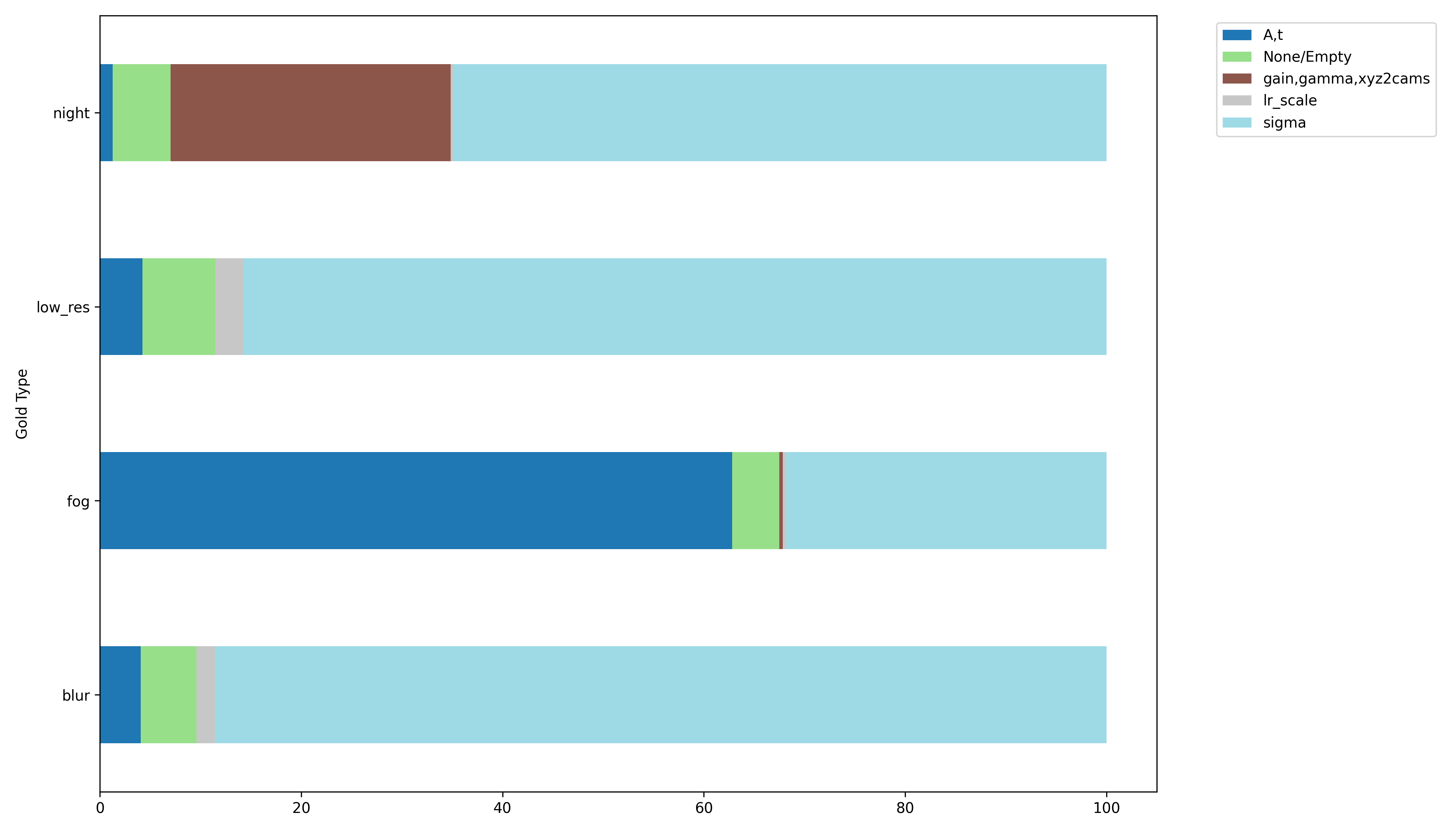}
        \caption{\textbf{GPT-4o}}
        \label{fig:param_gpt}
    \end{subfigure}
    \hfill
    \begin{subfigure}[b]{0.48\textwidth}
        \centering
        \includegraphics[width=\linewidth]{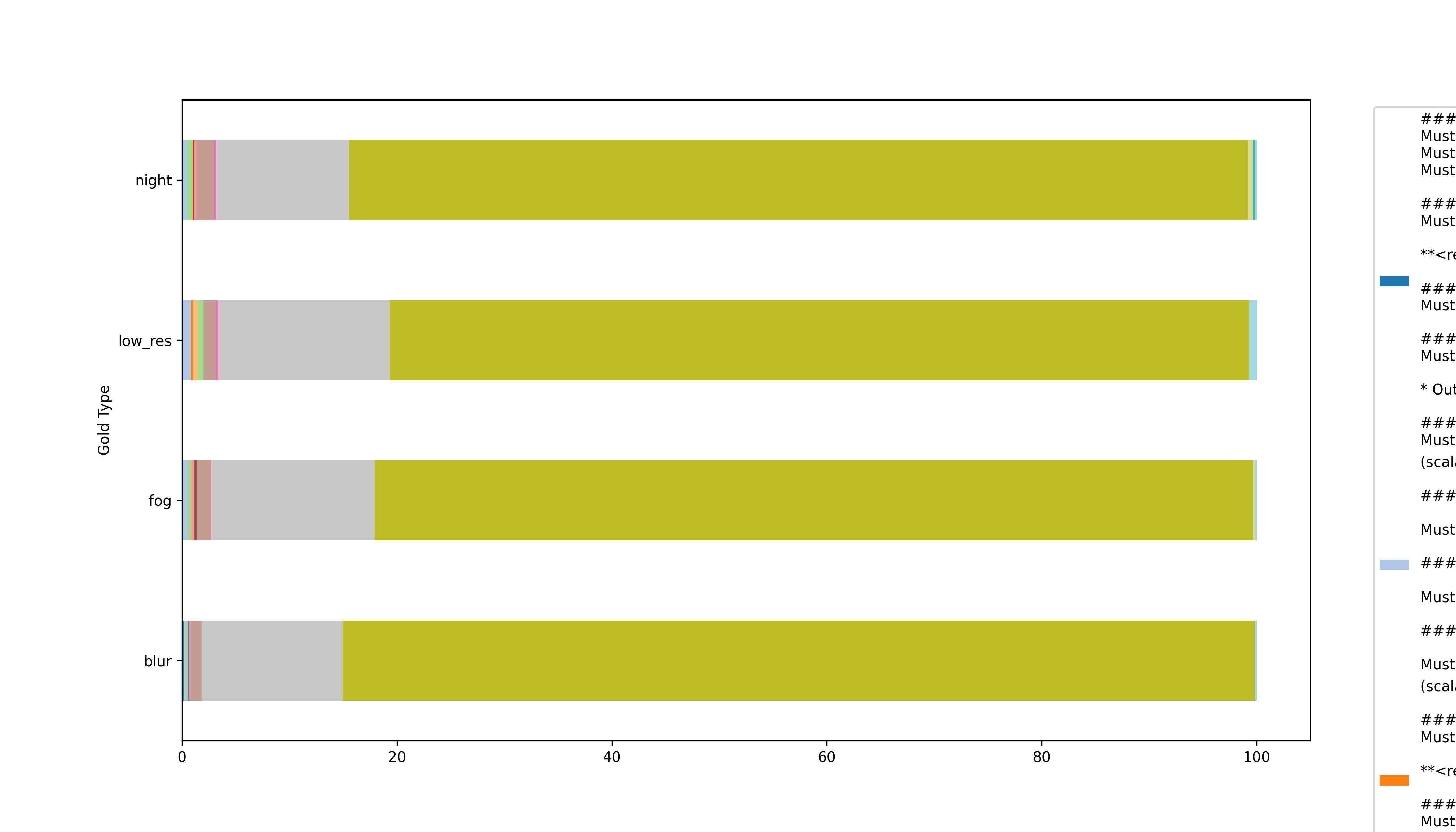}
        \caption{\textbf{LLaVA-1.5}}
        \label{fig:param_llava}
    \end{subfigure}
    
    \vspace{0.4cm}
    
    % Row 2: Qwen-8B & Qwen-32B
    \begin{subfigure}[b]{0.48\textwidth}
        \centering
        \includegraphics[width=\linewidth]{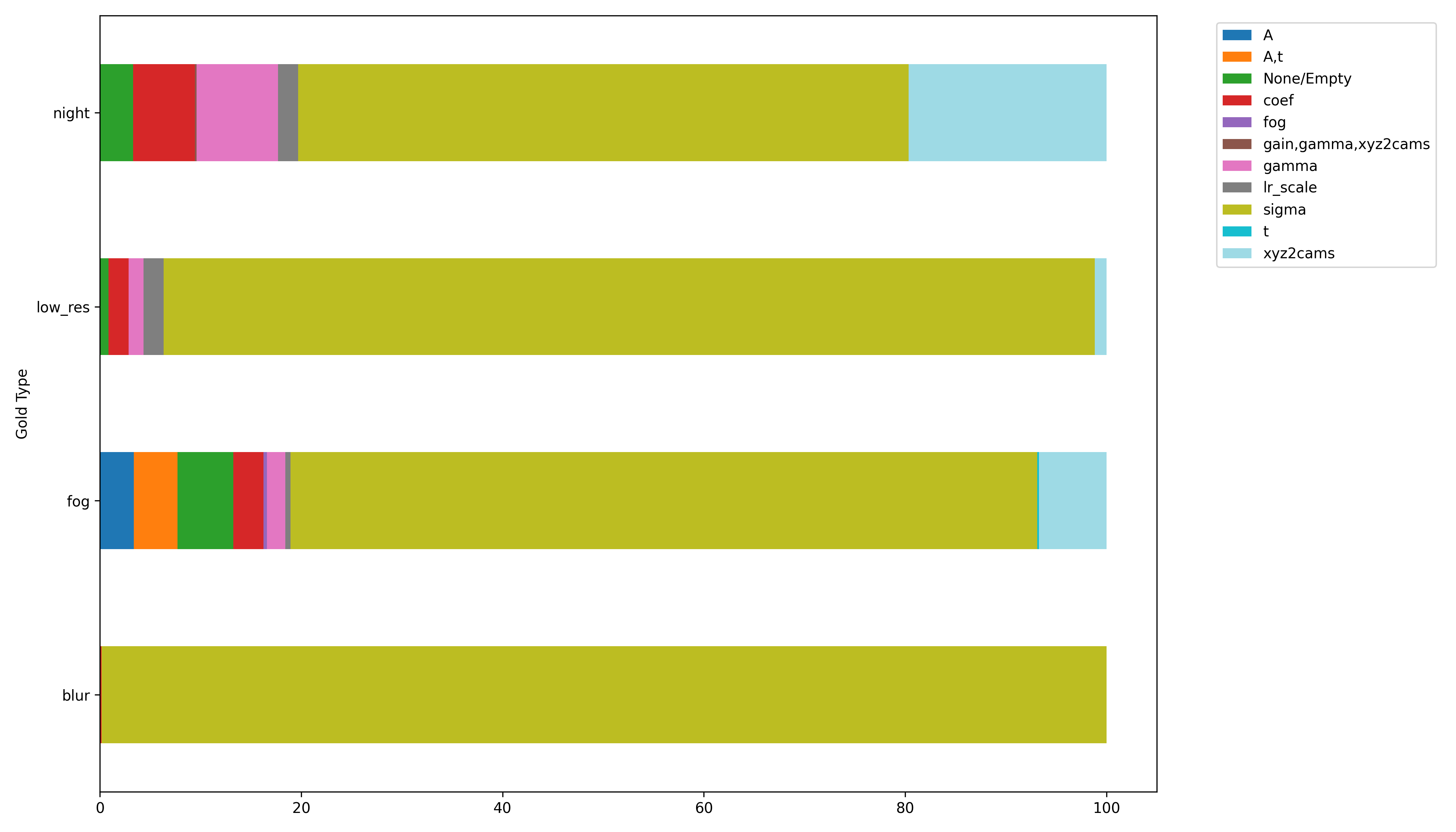}
        \caption{\textbf{Qwen-VL-8B}}
        \label{fig:param_qwen8b}
    \end{subfigure}
    \hfill
    \begin{subfigure}[b]{0.48\textwidth}
        \centering
        \includegraphics[width=\linewidth]{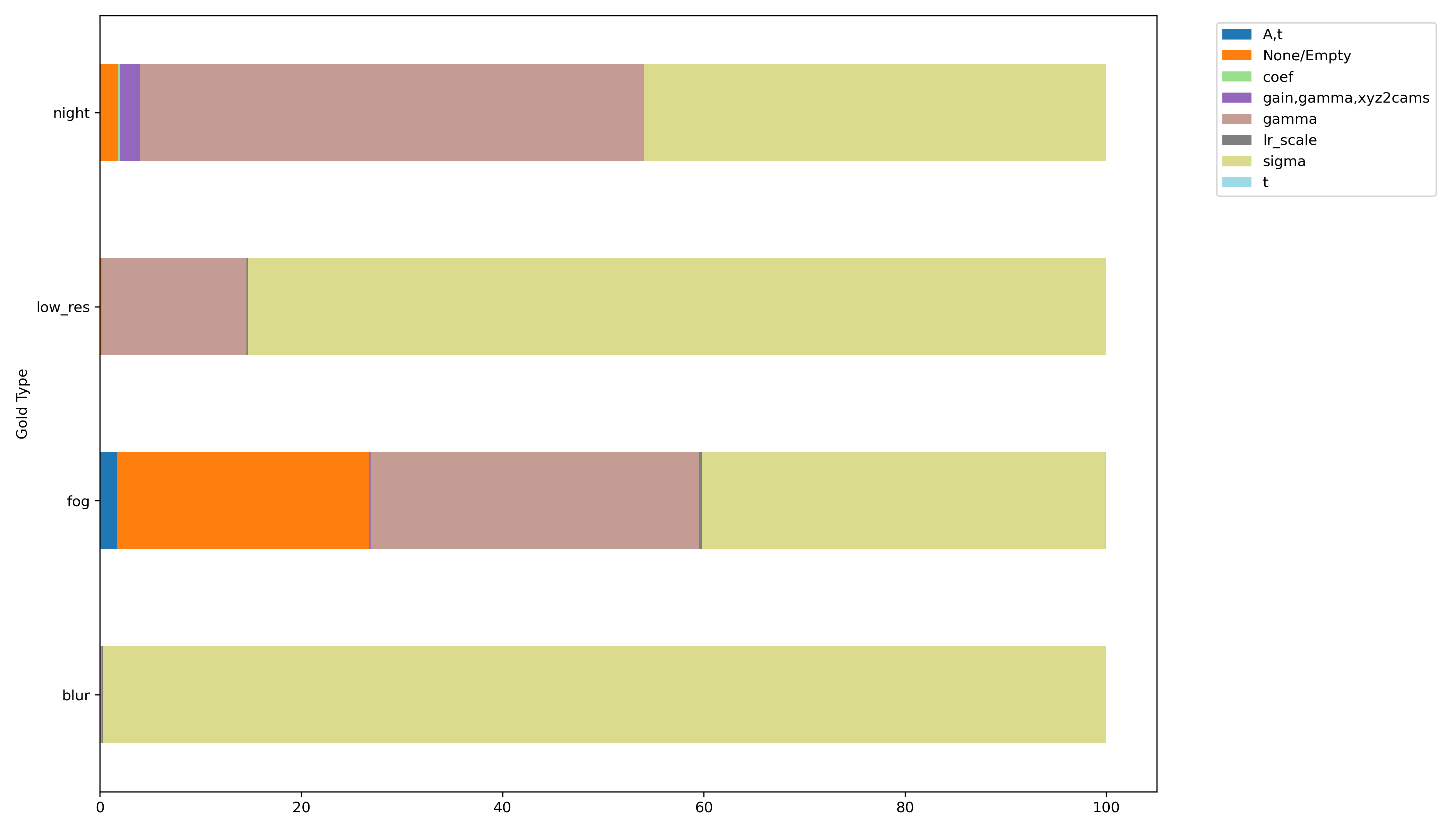}
        \caption{\textbf{Qwen-VL-32B}}
        \label{fig:param_qwen32b}
    \end{subfigure}

    \vspace{0.4cm}

    % Row 3: Qwen-235B (Centered)
    \begin{subfigure}[b]{0.48\textwidth}
        \centering
        \includegraphics[width=\linewidth]{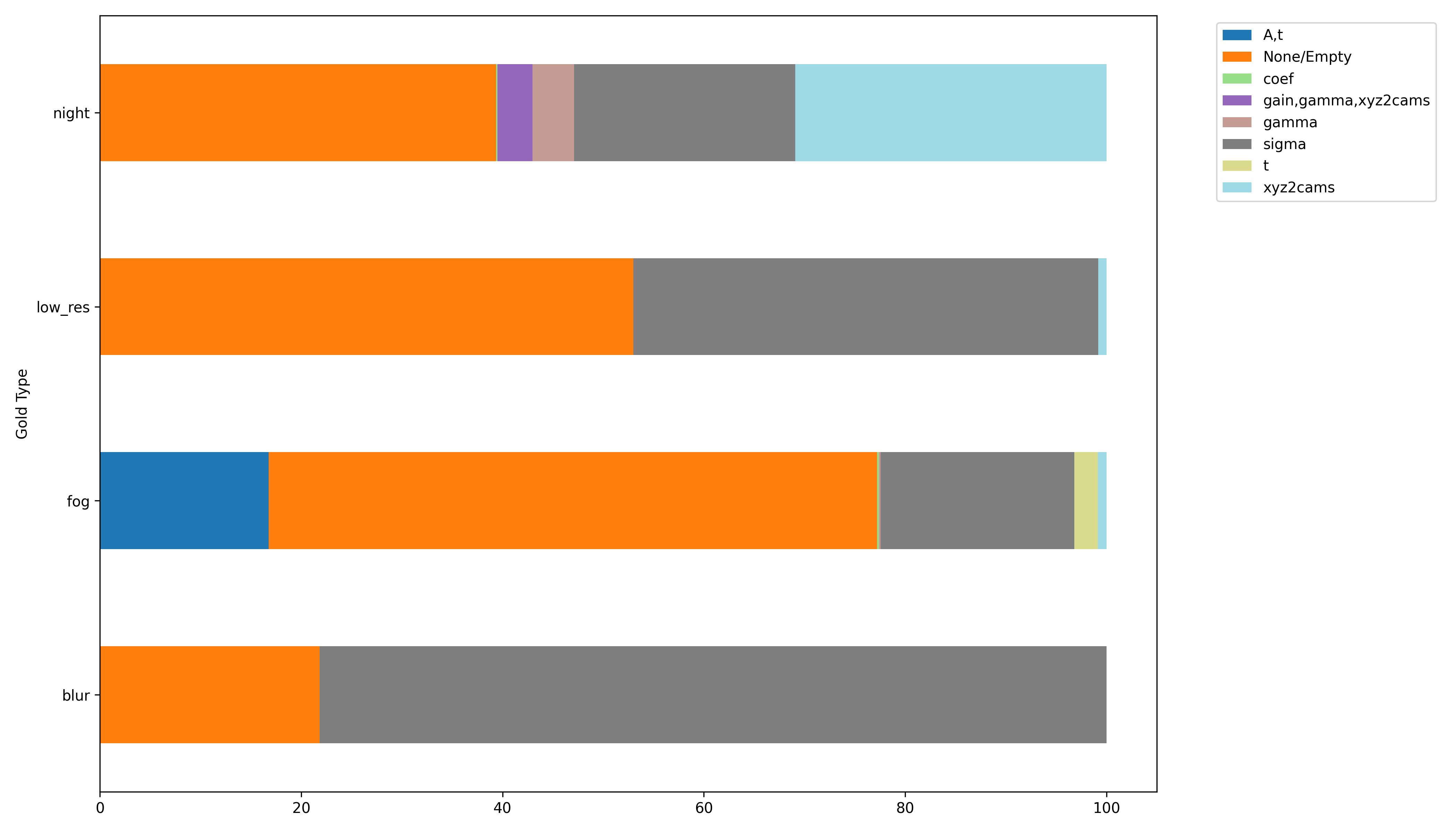}
        \caption{\textbf{Qwen-VL-235B}}
        \label{fig:param_qwen235b}
    \end{subfigure}
    \hfill
    \begin{subfigure}[b]{0.48\textwidth}
        \centering
        \includegraphics[width=\linewidth]{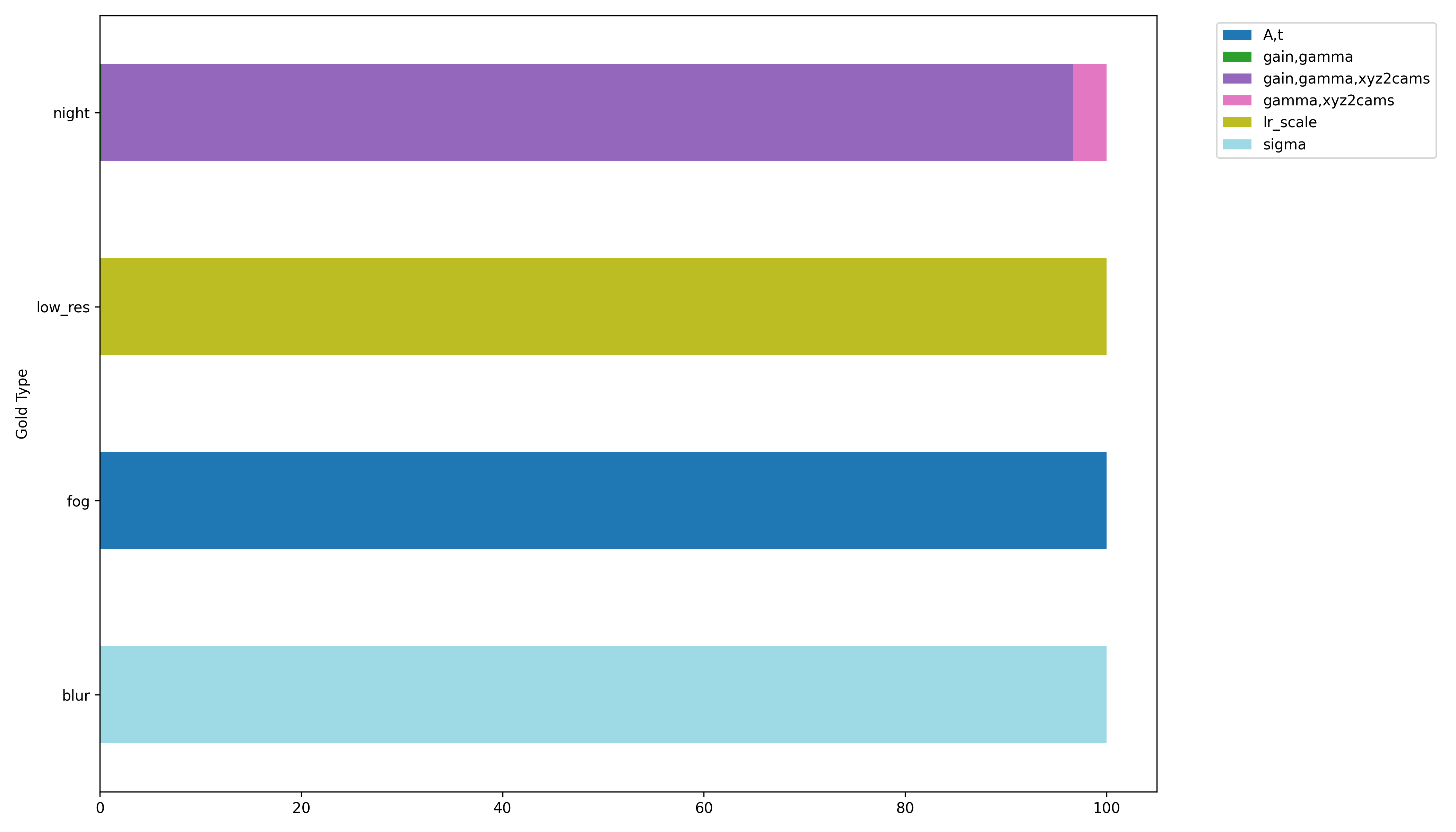}
        \caption{\textbf{DU-VLM}}
        \label{fig:param_du}
    \end{subfigure}

    \caption{\textbf{Parameter Key Distribution per Ground-Truth Type.} The x-axis represents the percentage of predicted keys. A physically grounded model should show distinct dominant colors for each row (e.g., \texttt{A, t} for Fog, \texttt{scale} for Low-Res). Instead, baselines show \textbf{Structural Hallucination}, repetitively predicting the same parameter (usually \texttt{sigma} or \texttt{gamma}) regardless of the actual physical degradation.}
    \label{fig:appendix_param}
\end{figure*}

\section{Implementation Details}
\paragraph{Implementation Details.}
We implement DU-VLM using the Qwen3-VL-8B architecture as the backbone. The model is trained on the DU-110k training set for 10 epochs using the AdamW optimizer with a learning rate of $1e-5$ and a cosine annealing schedule. For the restoration stage, we utilize the pre-trained DDNM \cite{wangzero}, injecting our predicted parameters $\hat{\mathcal{D}}$ into the reverse diffusion process. All experiments are conducted on $8 \times$ NVIDIA A100 GPUs.
\paragraph{Implementation Details of RL}
\label{sec:appendix_implementation_details}

The RL stage builds on the SFT-initialized policy to further improve the VLM's structured prediction capability, i.e., predicting the degradation type and a structured parameter tuple that conditions an external restoration prior. We train with GRPO, sampling multiple completions per input to form within-group samples for relative advantage estimation. The reward for each completion is produced by a deterministic evaluation pipeline: (i) parse the model output to obtain the predicted type $\hat{t}$ and parameters $\hat{\mathcal{D}}$; (ii) verify the validity of the parameter structure and value ranges; and (iii) invoke conditional restoration to produce the reconstructed image $\hat{I}_c$, which is then compared against the clean target $I_c$ using a pixel-space fidelity metric.

\paragraph{Gating strategy.}
We employ a strict gating mechanism to avoid assigning learning signal to invalid structures. If the predicted type mismatches the annotated label, or the predicted parameters fail type-specific format/shape/range checks, we immediately assign a fixed penalty $r=-1$ and skip restoration; only validated samples invoke DDNM to compute the reconstruction reward. This gating significantly reduces variance and prevents noisy rewards caused by querying the generative prior under malformed conditions.

\paragraph{Restoration prior and fidelity evaluation.}
After the predicted tuple passes validation, we use DDNM as a fixed conditional restoration prior $\mathcal{G}^{-1}$.
Given the degraded observation $I_d$ and the predicted degradation parameters $\hat{\mathcal{D}}$, the restored image is
\begin{equation}
\hat{I}_c = \mathcal{G}^{-1}(I_d;\hat{\mathcal{D}}).
\end{equation}
All images are resized to a fixed resolution (e.g., $256\times256$), and we apply the standard forward/inverse transforms
required by DDNM to map between image space and diffusion space.
We instantiate the fidelity metric $\mathcal{S}$ in Eq.~\eqref{eq:rl_reward_components} as the negative pixel-wise MSE:
\begin{equation}
r_{\mathrm{rec}}
=
\mathcal{S}(\hat{I}_c, I_c)
=
-\mathrm{MSE}(\hat{I}_c, I_c)
=
-\frac{1}{N}\sum_{i=1}^{N}\left(\hat{I}_c^{(i)}-I_c^{(i)}\right)^2,
\end{equation}
where $N$ denotes the number of elements (pixels $\times$ channels).
DDNM is always run in \texttt{eval} mode and serves only as a black-box evaluator; no gradients are backpropagated through the restoration process.

\paragraph{Online self-supervised variant.}
When clean targets are unavailable in the online phase, we reuse the same parsing and validation pipeline but replace the full-reference fidelity with a no-reference IQA score computed on $\hat{I}_c$. The model is updated iteratively and training stops when the running-average reward saturates. The same gating rules (type/parameter validity) are retained to ensure that IQA is only evaluated on semantically meaningful restorations, thereby reducing noise and misleading reward signals.

\begin{figure*}[htbp]
    \centering
    % 统一设置：所有子图宽度均为 0.24\textwidth
    
    % --- 第一行：4张图，Blur, Low Res, Haze (A, t) ---
    \begin{subfigure}[b]{0.24\textwidth}
        \centering
        \includegraphics[width=\textwidth]{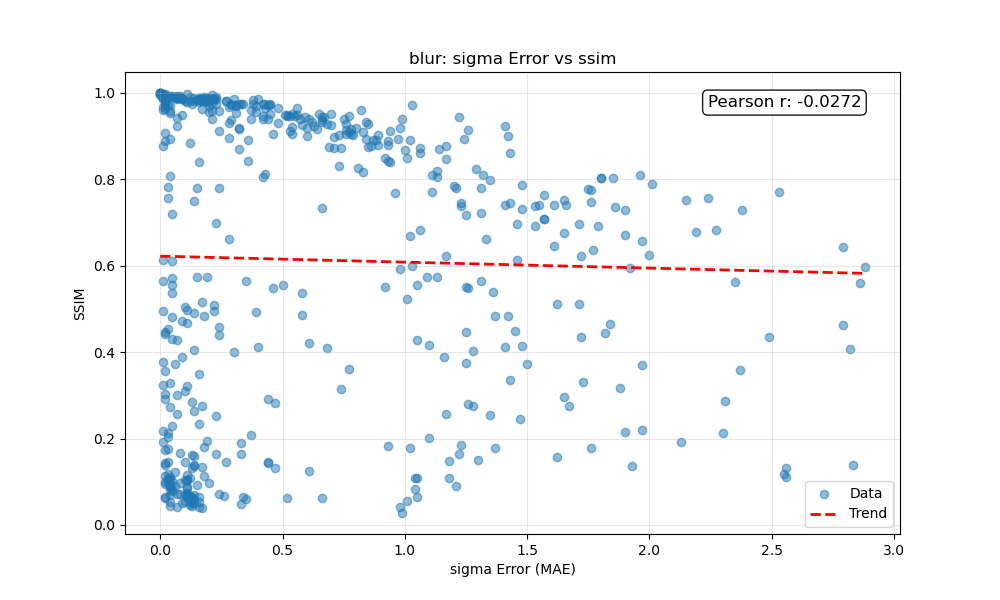}
        \caption{Blur: $\sigma$ Error}
        \label{fig:blur_ssim}
    \end{subfigure}
    \hfill
    \begin{subfigure}[b]{0.24\textwidth}
        \centering
        \includegraphics[width=\textwidth]{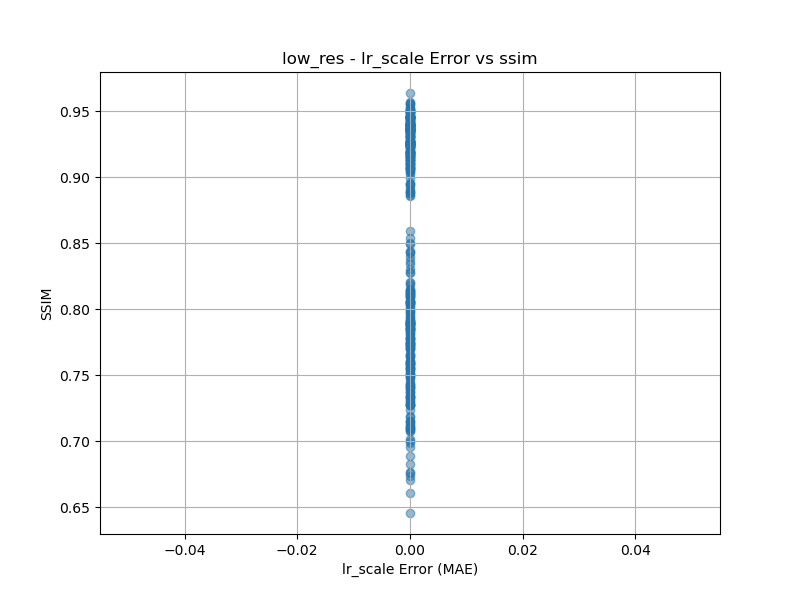}
        \caption{Low Res: Scale}
        \label{fig:low_res_ssim}
    \end{subfigure}
    \hfill
    \begin{subfigure}[b]{0.24\textwidth}
        \centering
        \includegraphics[width=\textwidth]{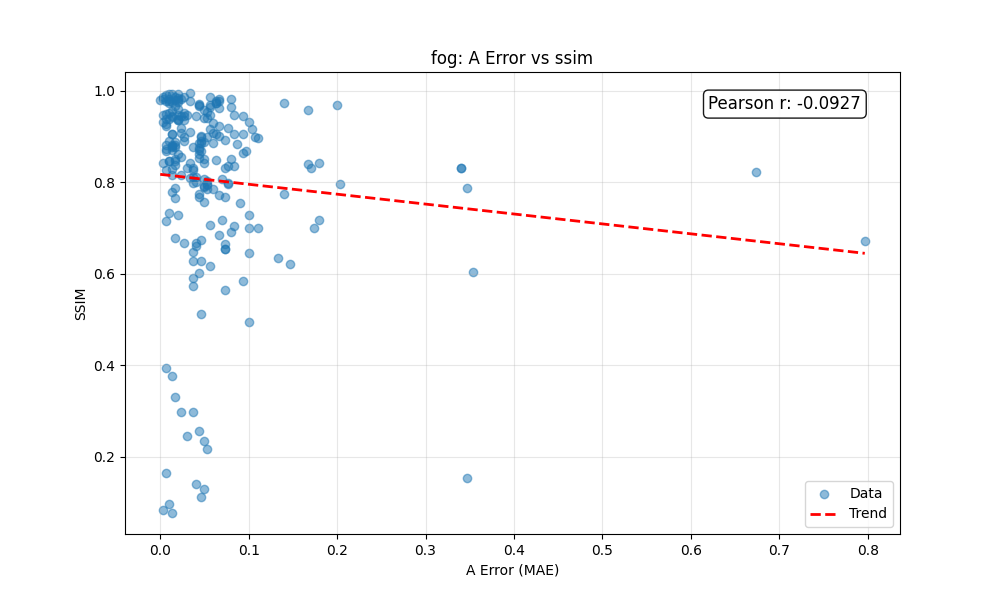}
        \caption{Haze: $A$ Error}
        \label{fig:haze_a_ssim}
    \end{subfigure}
    \hfill
    \begin{subfigure}[b]{0.24\textwidth}
        \centering
        \includegraphics[width=\textwidth]{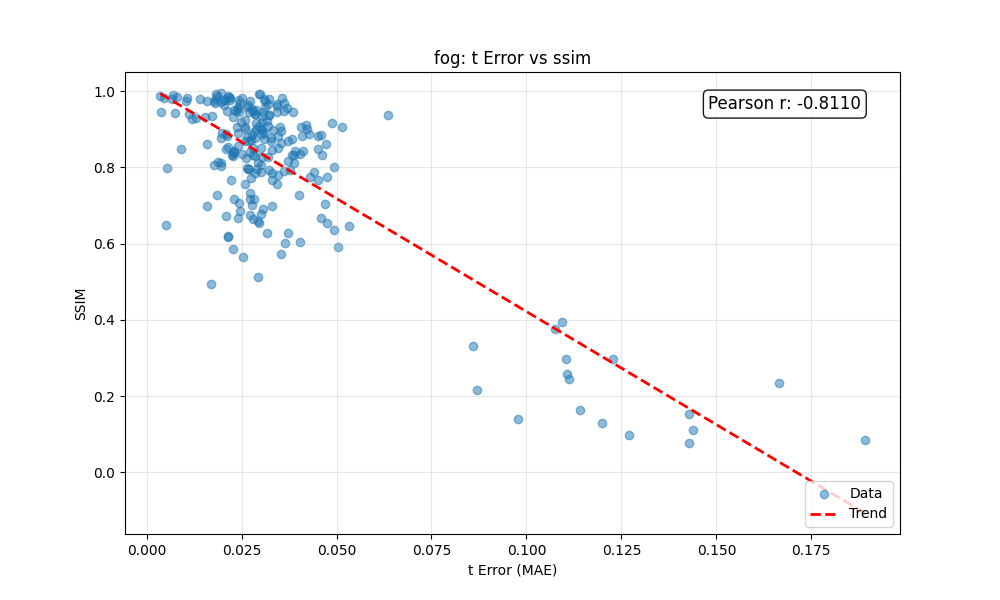}
        \caption{Haze: $t$ Error}
        \label{fig:haze_t_ssim}
    \end{subfigure}
    
    \vspace{1em} % 垂直间距
    
    % --- 第二行：3张图，Night (Gain, Gamma, XYZ) 居中 ---
    \begin{subfigure}[b]{0.24\textwidth}
        \centering
        \includegraphics[width=\textwidth]{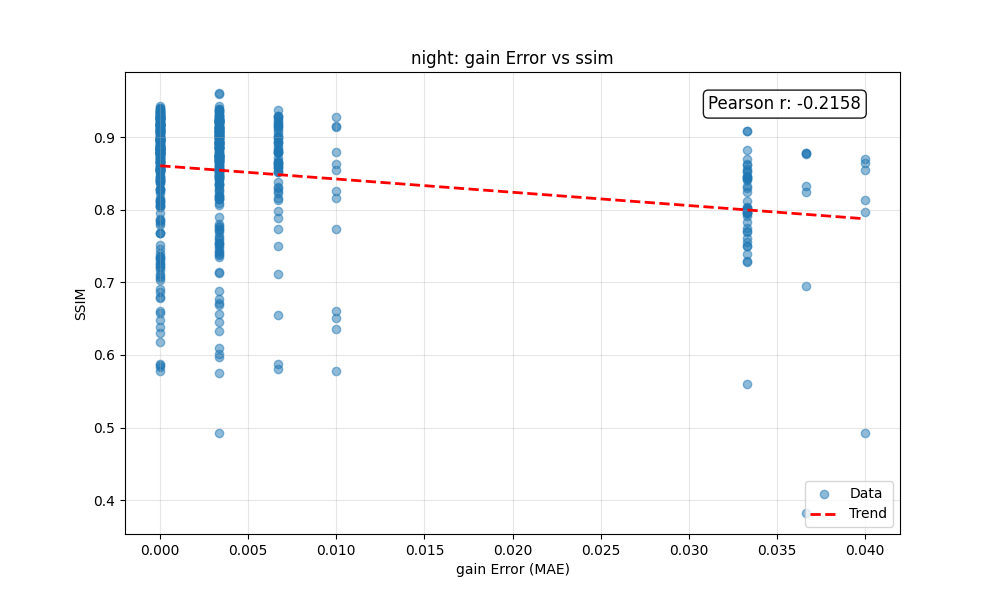}
        \caption{Night: Gain Error}
        \label{fig:night_gain_ssim}
    \end{subfigure}
    \hspace{0.5em}
    \begin{subfigure}[b]{0.24\textwidth}
        \centering
        \includegraphics[width=\textwidth]{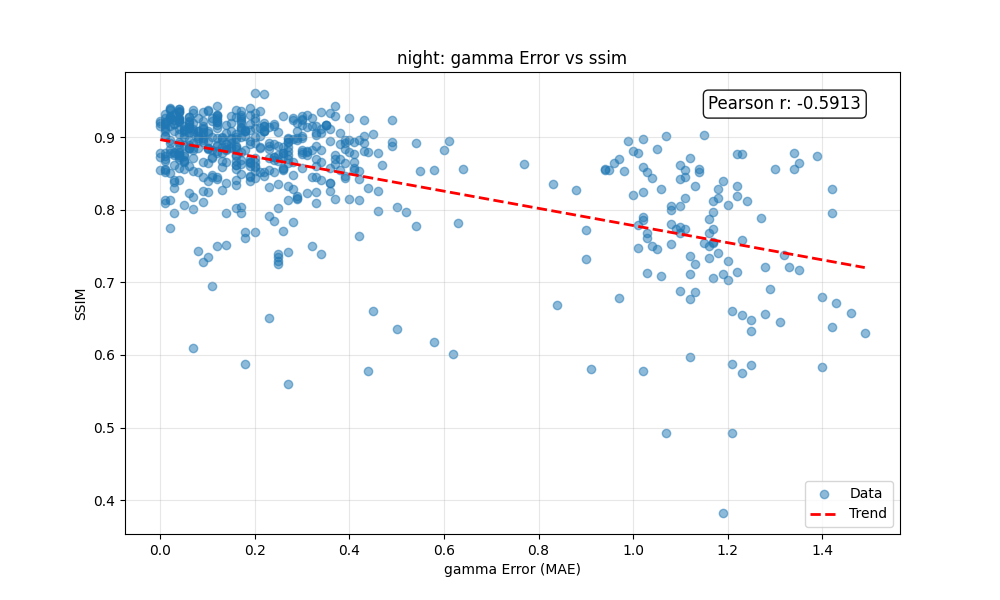}
        \caption{Night: $\gamma$ Error}
        \label{fig:night_gamma_ssim}
    \end{subfigure}
    \hspace{0.5em}
    \begin{subfigure}[b]{0.24\textwidth}
        \centering
        \includegraphics[width=\textwidth]{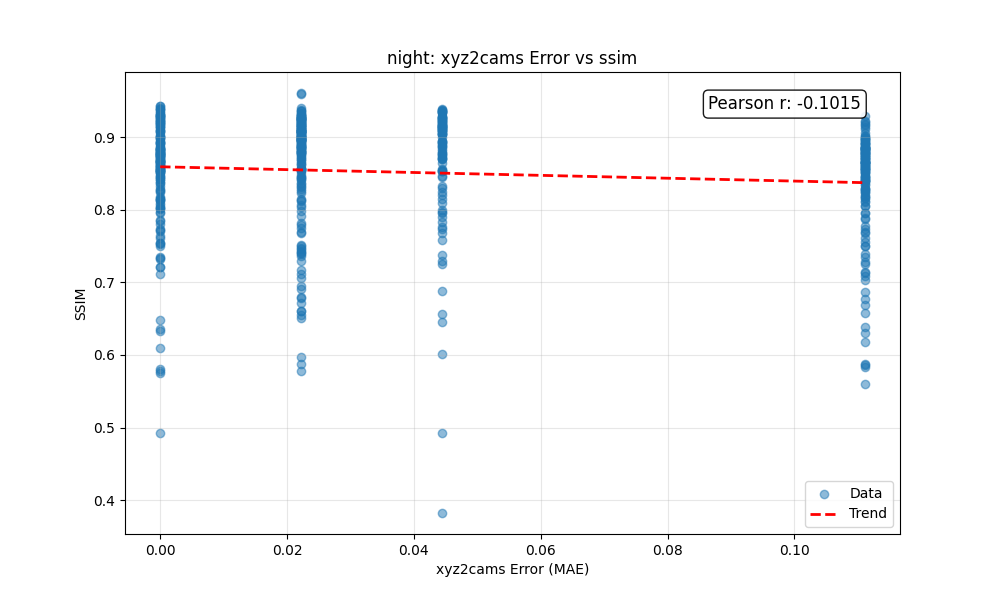}
        \caption{Night: Camera Intrinsic Error}
        \label{fig:night_xyz_ssim}
    \end{subfigure}
    
    \caption{SSIM analysis against parameter estimation errors. Top row: Blur, Low Resolution, and Haze parameters. Bottom row: Night scene parameters.}
    \label{fig:parameter_errors_ssim}
\end{figure*}

\section{Restoration Tolerance Analysis: Structural Similarity (SSIM)}
\label{subsec:ssim_tolerance}

Complementing the PSNR analysis in the main text, we explicitly quantify the impact of parameter estimation errors on Structural Similarity (SSIM) in \cref{fig:parameter_errors_ssim}. The results corroborate the trends observed in PSNR, revealing that structural fidelity is highly sensitive to specific degradation parameters.
Notably, deviations in the \textit{Haze Transmission} map ($t$) exhibit a strong negative correlation with SSIM ($\rho = -0.811$), indicating that inaccuracies in depth estimation can catastrophically disrupt the structural coherence of the restored image. Similarly, \textit{Night Gamma} correction errors ($\gamma$) significantly degrade structural quality ($\rho = -0.551$), likely due to the nonlinear distortion of local contrast.
In contrast, the restoration model demonstrates remarkable robustness to \textit{Blur} kernel estimation errors ($\sigma$), yielding a negligible correlation of $\rho = -0.027$. This suggests that for deblurring tasks, the generative prior of the diffusion model can effectively hallucinate plausible high-frequency structures even when the estimated kernel size deviates slightly from the ground truth.

\begin{figure*}[htbp]
    \centering
    % 统一设置：所有子图宽度均为 0.24\textwidth
    
    % --- 第一行：4张图，Blur, Low Res, Haze (A, t) ---
    \begin{subfigure}[b]{0.24\textwidth}
        \centering
        \includegraphics[width=\textwidth]{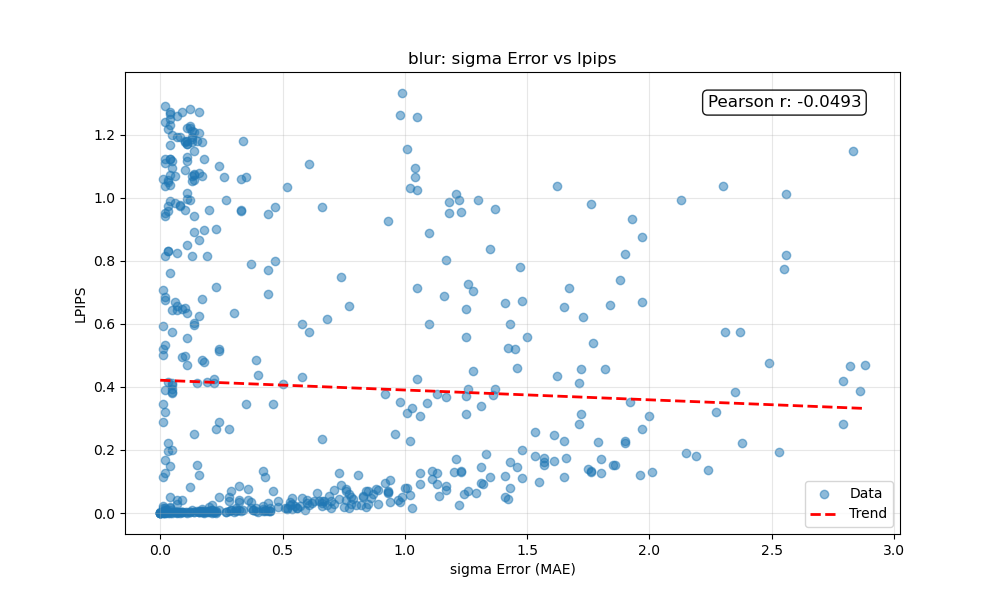}
        \caption{Blur: $\sigma$ Error}
        \label{fig:blur_lpips}
    \end{subfigure}
    \hfill
    \begin{subfigure}[b]{0.24\textwidth}
        \centering
        \includegraphics[width=\textwidth]{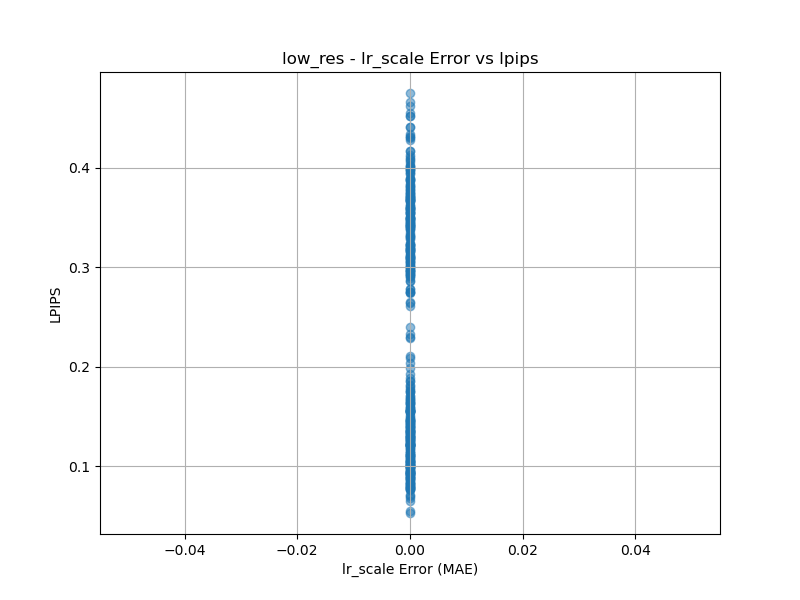}
        \caption{Low Res: Scale}
        \label{fig:low_res_lpips}
    \end{subfigure}
    \hfill
    \begin{subfigure}[b]{0.24\textwidth}
        \centering
        \includegraphics[width=\textwidth]{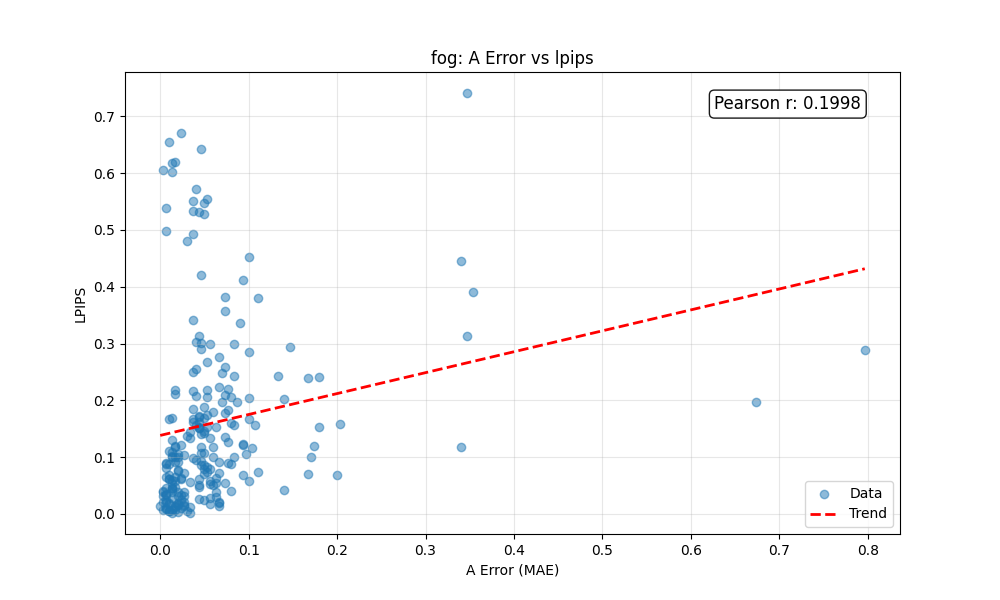}
        \caption{Haze: $A$ Error}
        \label{fig:haze_a_lpips}
    \end{subfigure}
    \hfill
    \begin{subfigure}[b]{0.24\textwidth}
        \centering
        \includegraphics[width=\textwidth]{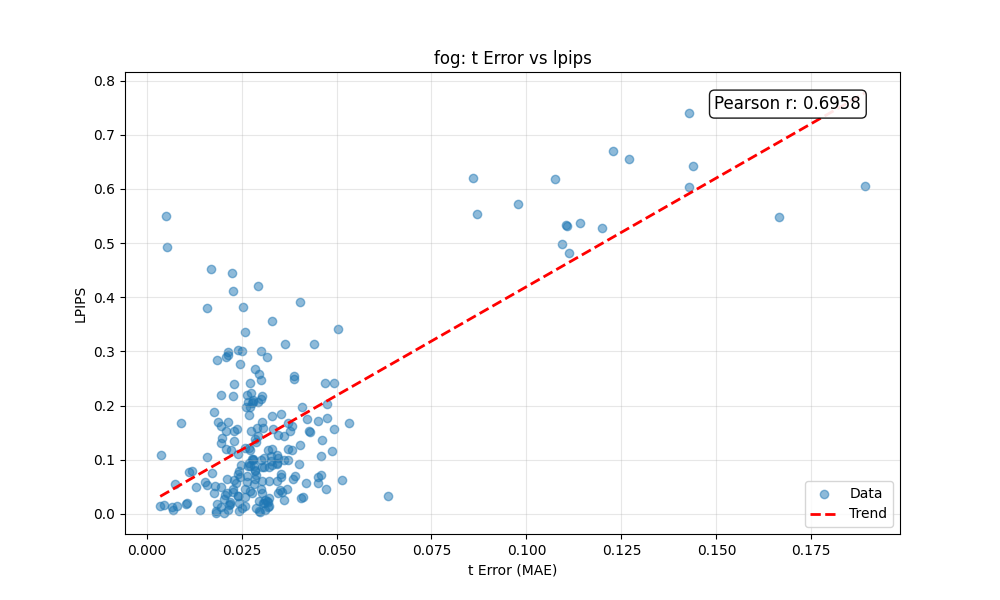}
        \caption{Haze: $t$ Error}
        \label{fig:haze_t_lpips}
    \end{subfigure}
    
    \vspace{1em} % 垂直间距
    
    % --- 第二行：3张图，Night (Gain, Gamma, XYZ) 居中 ---
    \begin{subfigure}[b]{0.24\textwidth}
        \centering
        \includegraphics[width=\textwidth]{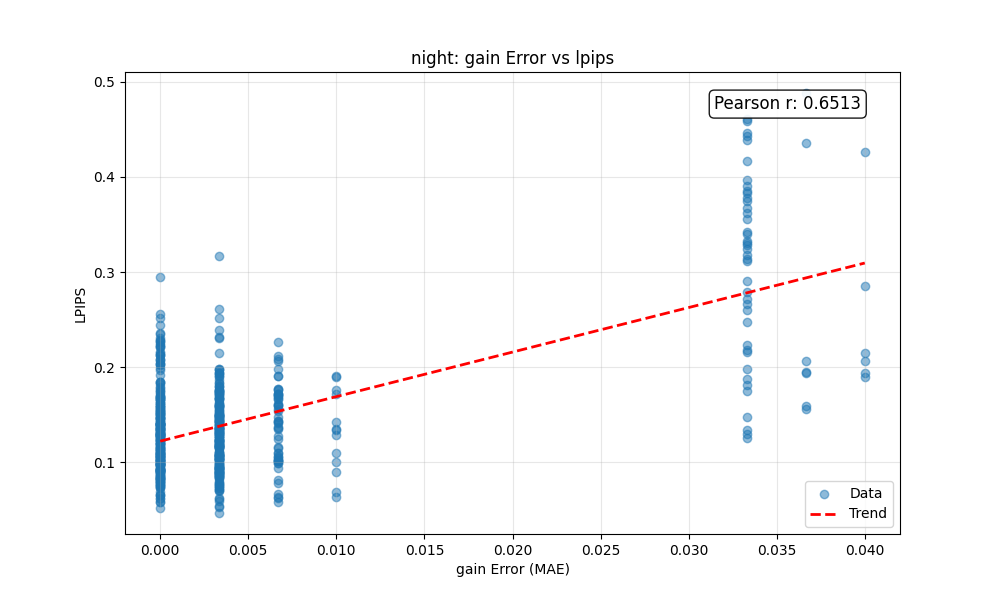}
        \caption{Night: Gain Error}
        \label{fig:night_gain_lpips}
    \end{subfigure}
    \hspace{0.5em}
    \begin{subfigure}[b]{0.24\textwidth}
        \centering
        \includegraphics[width=\textwidth]{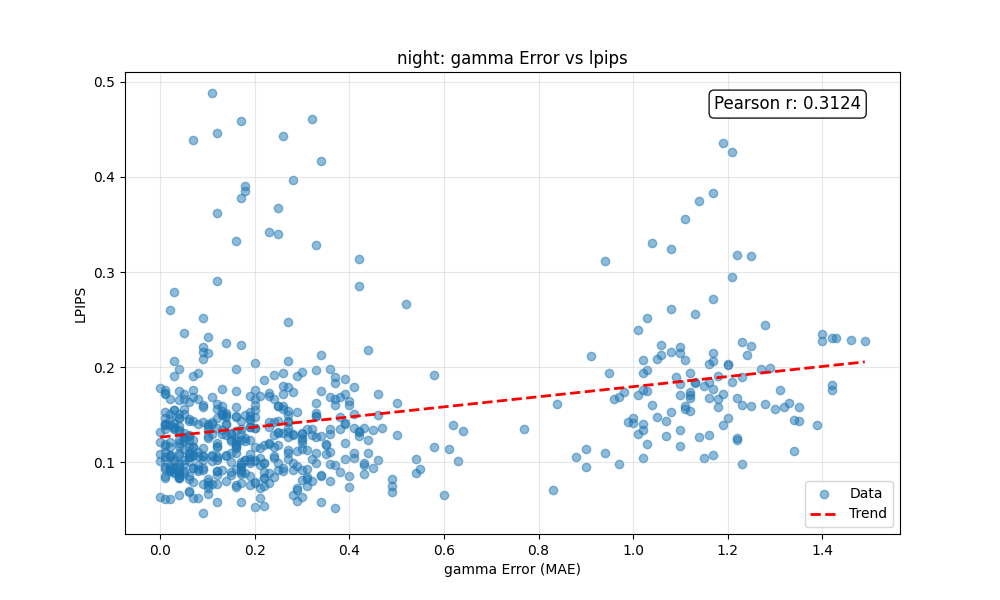}
        \caption{Night: $\gamma$ Error}
        \label{fig:night_gamma_lpips}
    \end{subfigure}
    \hspace{0.5em}
    \begin{subfigure}[b]{0.24\textwidth}
        \centering
        \includegraphics[width=\textwidth]{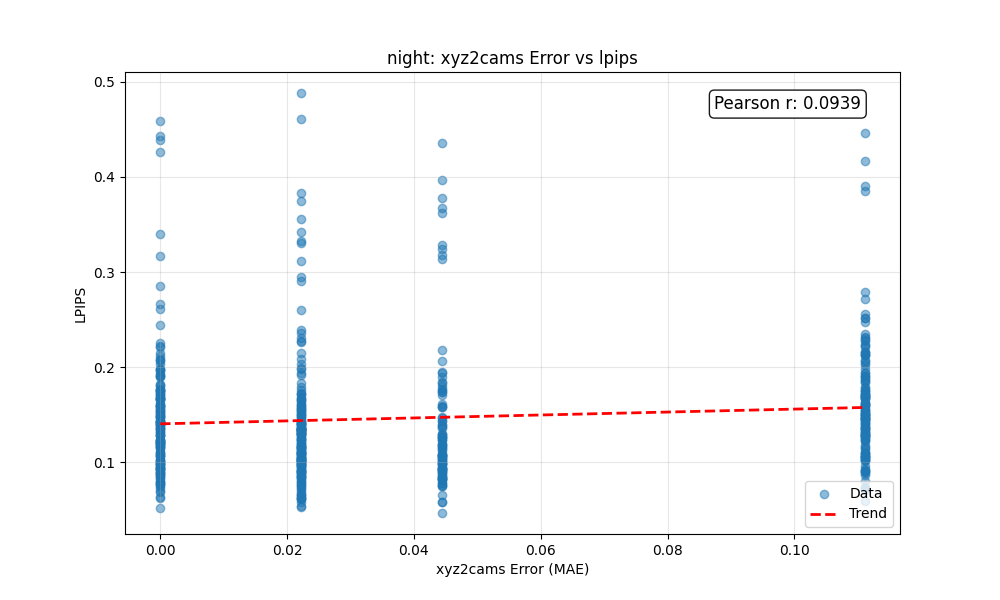}
        \caption{Night: Camera Intrinsic Error}
        \label{fig:night_xyz_lpips}
    \end{subfigure}
    
    \caption{LPIPS analysis against parameter estimation errors. Top row: Blur, Low Resolution, and Haze parameters. Bottom row: Night scene parameters.}
    \label{fig:parameter_errors_lpips}
\end{figure*}

\section{Restoration Tolerance Analysis: Perceptual Similarity (LPIPS)}
\label{subsec:lpips_tolerance}

We further assess the perceptual consistency of the restoration using Learned Perceptual Image Patch Similarity (LPIPS), where lower scores indicate better perceptual quality. Consequently, a positive correlation with estimation error implies a degradation in performance.
As illustrated in \cref{fig:parameter_errors_lpips}, the \textit{Haze Transmission} parameter ($t$) again dictates the upper bound of performance, showing the highest sensitivity ($\rho = 0.855$). This confirms that physical inaccuracies in atmospheric scattering models are perceptually jarring to the human eye.
Furthermore, \textit{Night Gain} errors show a more pronounced impact on LPIPS ($\rho = 0.651$) compared to their impact on SSIM ($\rho = 0.215$ in \cref{fig:parameter_errors_ssim}), suggesting that while gain errors may preserve edges (structure), they introduce perceptual artifacts such as amplified noise or unnatural color shifts.
Consistent with other metrics, \textit{Blur} parameters maintain a low correlation ($\rho \approx 0.05$), reinforcing the diffusion model's capability to act as a robust blind deblurring solver.